\documentclass{article}

\usepackage[utf8]{inputenc}
\usepackage[T1]{fontenc}
\usepackage{hyperref}
\usepackage{url}
\usepackage{booktabs}
\usepackage{amsfonts}
\usepackage{nicefrac}
\usepackage{microtype}
\usepackage{amsthm}
\newtheorem{theorem}{Theorem}

\usepackage{amsmath,amssymb}
\usepackage{algpseudocode}
\usepackage{algorithm}
\usepackage{xcolor}
\usepackage{bm}
\usepackage{bbm}
\usepackage{graphicx}
\usepackage{lmodern}
\usepackage{fullpage}
\usepackage[numbers,sort&compress]{natbib}
\usepackage{authblk}

\graphicspath{{./figures/}}

\title{Streaming Bayesian inference: theoretical limits and \\
    mini-batch approximate message-passing}
\author[1]{Andre Manoel}
\author[2]{Florent Krzakala}
\author[3]{Eric W. Tramel}
\author[4]{Lenka Zdeborová}
\affil[1]{Neurospin, CEA, Université Paris-Saclay}
\affil[2]{LPS ENS, CNRS, PSL, UPMC \& Sorbonne Univ.}
\affil[3]{OWKIN}
\affil[4]{IPhT, CNRS, CEA, Université Paris-Saclay}

\begin{document}

\maketitle

\begin{abstract}
    In statistical learning for real-world large-scale data problems, one must
    often resort to ``streaming'' algorithms which operate sequentially on
    small batches of data. In this work, we present an analysis of the
    information-theoretic limits of mini-batch inference in the context of
    generalized linear models and low-rank matrix factorization. In a
    controlled Bayes-optimal setting, we characterize the optimal
    performance and phase transitions as a function of mini-batch size. We
    base part of our results on a detailed analysis of a mini-batch version
    of the approximate message-passing algorithm (Mini-AMP), which we
    introduce. Additionally, we show that this theoretical optimality
    carries over into real-data problems by illustrating that Mini-AMP is
    competitive with standard streaming algorithms for clustering.
\end{abstract}

\section{Introduction}
\label{sec:intro}
In current machine learning applications, one often faces the challenge of
\emph{scale:} massive data causes algorithms to explode in time and memory
requirements. In such cases, when it is infeasible to process the full
dataset simultaneously, one must resort to "online" or "streaming" methods
which process only a small fraction of data points at a time --- producing a
step-by-step learning process. Such procedures are becoming more and more
necessary to cope with massive datasets. For example, one can see the
effectiveness of such approaches in deep learning via the stochastic
gradient descent algorithm \cite{bottou2010large} or in statistical
inference via the stochastic variational inference framework
\cite{hoffman_stochastic_2013}.

In this work, we treat streaming inference within a Bayesian framework
where, as new data arrives, posterior beliefs are updated according to
Bayes' rule. One well known approach in this direction is assumed density
filtering (ADF) \cite{opper_bayesian_1998,minka_expectation_2001}, which
processes a single data point at a time, a procedure to which we refer to as
\emph{fully online}. A number of other works analyzed various related fully
online algorithms \cite{mitliagkas2013memory,wang_online_2016},
especially in the statistical physics literature 
\cite{kinouchi1992optimal,biehl1994online,solla_optimal_1998,saad1999online,rossi_bayesian_2016}.
We are instead interested in the case where multiple samples -- a
\emph{mini-batch} -- arrive at once. Tuning the size of these mini-batches
allows us to to explore the trade-off between the precision and efficiency. 

Our motivation and setting are very much along the lines of streaming
variational Bayes (VB) inference \cite{broderick_streaming_2013}. With
respect to existing works, we bring three main contributions. (i) We
introduce a streaming algorithm based on approximate message passing (AMP)
\cite{donoho_message-passing_2009,rangan_generalized_2011,zdeborova_statistical_2016}
that we call Mini-AMP. As AMP treats some of the correlations which VB
neglects, it is expected that AMP either outperforms or matches VB. (ii)
Unlike other general-purpose algorithms for Bayesian inference, such as
Gibbs sampling or VB, AMP possesses the state evolution method which
asymptotically describes the performance of the algorithm for a class of
generative models. We extend this state evolution analysis to Mini-AMP.
(iii) For these generative models, we also analyze the optimal streaming
procedure, within a class of procedures that retains only point-wise
marginals from one step to another, and characterize regions of parameters
where Mini-AMP reaches optimality.

\section{Problem setting}
\label{sec:setting}
Denoting the vector of $N$ values to be estimated by $\bm{x}$, the data
presented at step $k$ by $\bm{y}^{(k)}$, and the collection of all
previously presented data by $\mathcal{D}^{(k - 1)} = \{\bm{y}^{(1)}, \dots,
\bm{y}^{(k - 1)} \}$, the posterior distribution at step $k$ is given by
\begin{equation}
    P (\bm{x} | \bm{y}^{(k)}, \mathcal{D}^{(k - 1)}) = \frac{P (\bm{y}^{(k)} |
    \bm{x}) \, P (\bm{x} | \mathcal{D}^{(k - 1)})}{\int d\bm{x} \, P
    (\bm{y}^{(k)} | \bm{x}) \, P(\bm{x} | \mathcal{D}^{(k - 1)})}.
    \label{eq:bayes_rule}
\end{equation}
In other words, with each presentation of new data, the prior distribution
is updated with the posterior distribution derived from the previously
presented data. Directly implementing this strategy is seldom feasible as
the normalizing integral in \eqref{eq:bayes_rule} is intractable in general.
Additionally, to keep the memory requirements small, we would like consider
only the case where $O(N)$ parameters are passed from one step to the next.
With this restriction, one cannot carry over high-order correlations from
previous steps. Instead, following the strategy of
\cite{broderick_streaming_2013}, we resort to a factorized approximation of
the "prior" term of \eqref{eq:bayes_rule},
\begin{equation}
    P (\bm{x} | \mathcal{D}^{(k - 1)}) \approx Q^{(k-1)} (\bm{x}) \!=\!
    \prod_{i = 1}^N \! \mathcal{P}_i \! \big[P(\bm{x} | \bm{y}^{(k-1)},
    \mathcal{D}^{(k - 2)})\big],  
    \label{program}
\end{equation}
where $\mathcal{P}_i[\cdot]$ denotes the posterior marginals of parameter
$x_i$ at a given step. Computing the marginals exactly is still
computationally intractable for most models of interest. In the present
work, this program is carried out with a scheme that is asymptotically exact
for a class of generative models {\it and} that has the advantage of being
amenable to a rigorous analysis.

We leverage the analysis of these models already conducted in the offline
setting using concepts and techniques from statistical physics
\cite{mezard_information_2009,donoho_message-passing_2009,rangan_generalized_2011,zdeborova_statistical_2016,rangan_iterative_2012,krzakala_probabilistic_2012,lesieur_phase_2016,lesieur2017statistical}
which have now been made almost entirely rigorous
\cite{bayati_dynamics_2011,bayati_universality_2015,barbier_mutual_2016-1,barbier2017mutual_last,reeves_replica-symmetric_2016,miolane2017fundamental}.
We show in particular that -- just as for the offline setting -- phase
transitions exist for mini-batch learning problems, and that their
description provides information about the learning error that is achievable
information-theoretically or computationally efficiently.

\section{Generative models and offline learning}
\label{sec:models}
In our theoretical analysis, we consider inference in popular models with
synthetic data generated from a given distribution, such as the perceptron
with random patterns \cite{gardner_three_1989,gyorgyi_first-order_1990},
sparse linear regression with a random matrix (compressed sensing)
\cite{candes_stable_2006} and clustering random mixtures of Gaussians
\cite{lesieur_phase_2016}.
For clarity, we restrict our presentation to the generalized linear models
(GLMs), focusing on sparse linear estimation. Our results, however, can be
extended straightforwardly to any problem where AMP can be applied. For
offline GLMs, the joint distribution of the observation $\bm{y} \in
\mathbb{R}^M$ and the unknown $\bm{x} \in \mathbb{R}^N$ is given by 
\begin{equation}
    P (\bm{y}, \bm{x} | \Phi) = \prod_{\mu = 1}^M P (y_\mu | z_\mu \equiv
        \bm{\Phi}_\mu \cdot \bm{x}) \, \prod_{i = 1}^N P_X (x_i)\, .
    \label{eq:model1}
\end{equation}
where $\bm{\Phi}_\mu$ is the $\mu$-th line of the $M \times N$ matrix
$\Phi$. We consider the situation where $\Phi$ is a random matrix where
each element is taken i.i.d. from ${\cal {N}}(0,1/N)$ and $\alpha = M/N$.
Structured matrices have also been studied with AMP
\cite{rangan2016vector,cakmak2014s}. Two situations of interest described by
GLMs are (a) sparse linear regression (SLR) where the likelihood is Gaussian
$P (y_\mu | z_\mu) = \mathcal{N} (y_\mu ; z_\mu, \Delta)$ and the parameters
are sparse, for instance drawn from a Gauss-Bernoulli distribution $P_X(x_i)
= \rho \, \mathcal{N} (x_i; 0, 1) + (1 - \rho) \, \delta (x_i)$, and (b) the
probit regression problem $P (y_\mu | z_\mu) = \frac{1}{2}
\operatorname{erfc} \big( -\frac{y_\mu z_\mu}{\sqrt{2 \Delta}} \big)$ that
reduces to the perceptron $P (y_\mu | z_\mu) = \theta (y_\mu z_\mu)$ when
$\Delta\!\to \!0$
\cite{gardner_optimal_1988,gyorgyi_first-order_1990,zdeborova_statistical_2016}.
We first summarize the known relevant results for the fully {\it offline}
learning problem, where one processes all data at once. Again, for clarity,
we focus on the case of SLR. 

The marginals estimated by AMP are given by
\cite{donoho_message-passing_2009,rangan_generalized_2011,krzakala_probabilistic_2012}
\begin{equation}
    P (x_i | \Phi, \bm{y}) \approx q (x_i | A, B_i) = P_X (x_i) \,
        e^{-\frac{1}{2} A x_i^2 + B_i x_i} \big/ {Z (A, B_i)} \,,
\label{eq:post-AMP}
\end{equation}
where $Z(A, B_i)$ is a normalization factor. We shall denote the mean and
variance of this distribution by $\eta (A, B) \equiv
\frac{\partial}{\partial B} \log Z (A, B)$ and $\eta' (A, B) \equiv
\frac{\partial}{\partial B} \eta (A, B)$. The mean, in particular, provides
an approximation to the minimum mean-squared error (MMSE) estimate of
$\bm{x}$. The AMP iteration reads
\begin{equation}
    \bm{z}^{(t)} = \bm{y} - \Phi \hat{\bm{x}}^{(t)} +
    \alpha^{-1} \bm{z}^{(t - 1)} \; A^{(t - 1)} V^{(t)},
    \label{eq:amp1} 
\end{equation}
\begin{align}
    &\bm{B}^{(t)} = A^{(t)} \hat{\bm{x}}^{(t)} + A^{(t)} \; \alpha^{-1}
        \Phi^T \bm{z}^{(t)}, &\quad
    &{A}^{(t)} = \frac{\alpha}{\Delta + {V}^{(t)}},
    \label{eq:amp2} \\
    &\hat{x}_i^{(t + 1)} = \eta (A^{(t)}, B_i^{(t)}) \; \forall i, &\quad
    &V^{(t + 1)} = {\textstyle\frac{1}{N} \sum_{i = 1}^N} \,
        \eta' (A^{(t)}, B_i^{(t)}).
    \label{eq:amp3}
\end{align}

One of the main strengths of AMP is that when the matrix $\Phi$ has i.i.d.
elements, the ground truth parameters are generated i.i.d. from a
distribution $P_0(x_i)$, and $P(y_\mu|z_\mu) = {\cal N}(y_\mu; z_\mu,
\Delta_0)$, then the behavior and performance of AMP can be studied
analytically in the large system limit ($N\rightarrow\infty$) using a
technique called state evolution. This was proven by
\cite{bayati_dynamics_2011} who show that in the large $N$ limit, $A^{(t)}$
and $B_i^{(t)}$ converge in distribution such that, defining ${\cal E}^{(t)}
\equiv \mathbb{E} \, \big(\eta ({\cal A}^{(t - 1)}, {\cal B}^{(t - 1)}) -
x\big)^2$, with $x \sim P_0(x)$, and ${\cal V}^{(t)} \equiv \mathbb{E} \,
\eta' ({\cal A}^{(t - 1)}, {\cal B}^{(t - 1)})$, one has
\begin{equation}
    A^{(t)} \rightsquigarrow {\cal A}^{(t)} = \frac{\alpha}{\Delta + {\cal V}^{(t)}},
    \qquad
    B_i^{(t)} \rightsquigarrow {\cal B}^{(t)} \sim \mathcal{N} \left({\cal
        A}^{(t)} x, \alpha \frac{\Delta_0 + {\cal E}(t)} {(\Delta + {\cal
        V}(t))^2}\right) \, .
    \label{bayati}
\end{equation}
The behavior of the algorithm is monitored by the computation of the scalar
quantities ${\cal E}^{(t)}$ and ${\cal V}^{(t)}$. 

The "Bayes-optimal" setting is defined as the case when the generative model
is known and matches the terms in the posterior (\ref{eq:bayes_rule}), i.e.
when $P_X=P_0$, and $\Delta=\Delta_0$. One can show that in this case
${\cal E}^{(t)}={\cal V}^{(t)}$ (the so-called Nishimori property
\cite{zdeborova_statistical_2016}), so that the state evolution further
reduces to
\begin{equation}
    {\cal A}^{(t)} = \frac{\alpha}{\Delta + {\cal E}^{(t)}}, \quad
    {\cal E}^{(t)} = \mathbb{E} \eta' ({\cal A}^{(t-1)}, {\cal A}^{(t-1)} x
        + \sqrt{{\cal A}^{(t-1)}} z),
    \label{eq:se}
\end{equation}
with $x\!\sim\!P_X (x)$, $z\!\sim\!\mathcal{N} (0, 1)$, and ${\cal E}^{(t)}$
is the mean-squared error (MSE) achieved at iteration $t$.

Another set of recent results
\cite{barbier2017mutual_last,reeves_replica-symmetric_2016} allows for the
exact computation of the Bayes-optimal MMSE and the mutual information
between the observations and the unknown parameters. Given model
(\ref{eq:model1}) with Gaussian likelihood, the mutual information per
variable is given by the minimum of the so-called replica mutual
information: $ \lim_{N \to \infty} I({\bm X},{\bm Y}) = {\rm min} ~
i_\text{RS} ({\cal E})$ where, defining, $\Sigma^{-2}({\cal E}) \equiv
\frac{\alpha}{\Delta+{\cal E}}$,
\begin{equation}
    i_\text{RS} ({\cal E}) = \frac{\alpha}{2} \left[ \frac{{\cal E}}{\Delta
        + {\cal E}} + \log{\left(1+ \frac{{\cal E}}{\Delta}\right)} \right] -
        \mathbb{E}_{x, z} \left[ \log \mathbb{E}_{\tilde x}
        e^{-\frac{\left(\tilde x-\left(x+ z \Sigma({\cal E}) \right)\right)^2}{2
        \Sigma^2({\cal E})}} \right]-\frac 12,
    \label{eq:phi}
\end{equation}
with $x \sim P_X (x), \tilde x \sim P_X (\tilde x)$ and $z \sim \mathcal{N}
(0, 1)$. The MMSE is then given by $\arg\min ~ i_\text{RS} ({\cal E})$.

Comparisons between the MMSE and the MSE provided by AMP after convergence
are very instructive, as shown in \cite{zdeborova_statistical_2016}.
Typically, for large enough noise, ${\cal E}^{(t\to \infty)}={\cal E}_{\rm
AMP}=\rm{MMSE}$ and AMP achieves the Bayes-optimal result in polynomial
time, thus justifying, {\it a posteriori}, the interest of such algorithms
in this setting. In fact, since the fixed points of the state evolution are
all extrema of the mutual information (\ref{eq:phi}), it is useful to think
of AMP as an algorithm that attempts to minimize (\ref{eq:phi}). However, a
\emph{computational} phase transition can exist at low noise levels, where
$i_\text{RS} ({\cal E})$ has more than a single minimum. In this case, it
may happen that AMP does not reach the \emph{global} minimum, and therefore
${\cal E}_{{\rm AMP}}>{\rm MMSE}$. It is a remarkable open problem to
determine whether finding the MMSE in this region is computationally
tractable. The results we have just described are not merely restricted to
SLR, but appear {\it mutatis mutandis} in various cases of low-rank matrix
and tensor
factorization~\cite{rangan_iterative_2012,lesieur_phase_2016,lesieur2017statistical,krzakala2016mutual,barbier_mutual_2016-1,miolane2017fundamental}
and also partly in
GLMs~\cite{rangan_generalized_2011,zdeborova_statistical_2016} (in GLMs the
replica mutual information is so far only conjectured).

\section{Main results}
\label{sec:results}
\subsection{Mini-AMP}
Our first contribution is the Mini-AMP algorithm, which adapts AMP to the
streaming setting. Again, we shall restrict the presentation to the linear
regression case. The adaptation to other AMP algorithms is straightforward.
We consider a dataset of $M$ samples with $N$ features each, which we split
into $B$ mini-batches, each containing $M_b = M / B$ samples. We denote
$\alpha = M/N$ and $\alpha_b = M_b / N$. Crucially, for each step, the
posterior marginal given by AMP \eqref{eq:post-AMP} is the prior multiplied
by a quadratic form. Performing the program discussed in \eqref{program} is
thus tractable as the ${\cal P}_i$ are given by a Gaussian distribution
multiplied by the original prior. The only modification w.r.t. the offline
AMP at each step is thus to update the prior by multiplying the former one
by the exponential in \eqref{eq:post-AMP}. In other words, we use the
following "effective" prior when processing the $k$-th mini-batch:
\begin{equation}
    P^{k}_ {\Lambda_{k-1}, \bm{\Theta}_{k-1}}(\bm{x}) = P_X (\bm{x}) \, \prod_{i = 1}^N e^{-\frac{1}{2} \Lambda_{k-1} x_i^2 + \Theta_{k-1, i} x_i},\, 
    \quad \text{where }
	\Lambda_{k-1}=\sum_{\ell=1}^{k-1} A_{l},~~\Theta_{k-1, i}=\sum_{\ell=1}^{k-1} B_{l,i}.
\label{eff-prior}
\end{equation}
In practice, the only change when moving from AMP to Mini-AMP is therefore
the update of the arguments of the $\eta$ function. After $k$ mini-batches
have been processed, one replaces \eqref{eq:amp3} by
\begin{equation}
    \begin{aligned}
        \hat{x}_{k, i}^{(t + 1)} &= \eta \bigg( \underbrace{{\textstyle\sum}_{\ell = 1}^{k - 1} A_{\ell}}_{\Lambda_{k - 1}} + A_{k}^{(t)},
        \underbrace{{\textstyle\sum}_{\ell = 1}^{k - 1} B_{\ell, i}}_{\Theta_{k - 1,i}} + B_{k, i}^{(t)}\bigg), \\
    V_{k}^{(t + 1)} &= {\textstyle \frac{1}{N} \sum_{i = 1}^{N} \,
            \eta' \bigg( \underbrace{{\textstyle\sum}_{\ell = 1}^{k - 1} A_{\ell}}_{\Lambda_{k - 1}} + A_{k}^{(t)},
        \underbrace{{\textstyle\sum}_{\ell = 1}^{k - 1} B_{\ell, i}}_{\Theta_{k - 1, i}} + B_{k, i}^{(t)}\bigg).}
    \end{aligned}
    \label{eq:miniamp}
\end{equation}
The corresponding pseudo-code is given as Algorithm $1$. Each Mini-AMP
iteration has a computational complexity proportional to $M_b \times N$. We
note that, in the fully online scheme when $M_b = 1$, Mini-AMP with a single
iteration performed per sample gives the same as ADF
\cite{opper_bayesian_1998,rossi_bayesian_2016}.

\begin{algorithm}[ht!]
    \caption{Mini-AMP}
    \algnewcommand\algorithmicto{\textbf{to}}
    \algrenewtext{For}[3]%
    {\algorithmicfor\ $#1 \gets #2$ \algorithmicto\ $#3$ \algorithmicdo}
    \begin{algorithmic}[1]
        \State initialize $\Lambda_0 = 0$, $\Theta_{0, i} = 0 \; \forall i$
        \For{k}{1}{B}
            \State initialize $z_{k, \mu}^{(1)} = 0 \; \forall \mu$
            \State initialize $\hat{x}_{k, i}^{(1)} = \eta (\Lambda_{k - 1}, \Theta_{k - 1, i}) \; \forall i,
                V_{k}^{(1)} = \frac{1}{N} \sum_{i = 1}^N \eta' (\Lambda_{k - 1}, \Theta_{k - 1, i})$
            \For{t}{1}{t_\text{max}}
                \State compute $\bm{z}_k^{(t)}$ using (\ref{eq:amp1})
                \State compute $A_k^{(t)}$, $\bm{B}_k^{(t)}$ using (\ref{eq:amp2})
                \State compute $V_k^{(t + 1)}$, $\hat{\bm{x}}_k^{(t + 1)}$ using (\ref{eq:miniamp})
            \EndFor
            \State accumulate $\Lambda_{k} \gets \Lambda_{k - 1} + A_k$
            \State accumulate $\bm{\Theta}_{k} \gets \bm{\Theta}_{k - 1} + \bm{B}_k$
        \EndFor
    \end{algorithmic}
\end{algorithm}

\subsection{State evolution}
\begin{theorem}[State evolution of Mini-AMP]\label{thm1} 
    For a random matrix $\Phi$, where each element is taken i.i.d. from ${\cal
    {N}}(0,1/N)$, the MSE of Mini-AMP can be monitored asymptotically
    ($N\!\to\!\infty$ while $\alpha_b\!=\!O(1)$) by iterating the following
    state evolution equations,
    \begin{equation}
        \begin{aligned}
            \lambda_{k}^{(t)} &= \lambda_{k - 1} + \frac{\alpha_\text{b}}{\Delta + {\cal V}_{k}^{(t)}}, &\quad\quad
            {\cal V}_{k}^{(t + 1)} &= \mathbb{E}_{x, z} \,
            \eta' \big( \lambda_{k}^{(t)}, \lambda_{k}^{(t)} x + \sqrt{\gamma_{k}^{(t)}} z \big), \\
            \gamma_{k}^{(t)} &= \gamma_{k - 1} + \alpha_\text{b} \frac{\Delta_0 + {\cal E}_k^{(t)}}{(\Delta + {\cal V}_k^{(t)})^2},
            &\quad\quad
            {\cal E}_{k}^{(t + 1)} &= \mathbb{E}_{x, z} \Big( \eta \big( \lambda_{k}^{(t)}, \lambda_{k}^{(t)} x + \sqrt{\gamma_{k}^{(t)}} z \big) - x \Big)^2 \, ,
        \end{aligned}
        \label{eq:minise}
    \end{equation}
    where $x \sim P_0(x)$, and $z \sim {\cal N}(0,1)$. For each $k = 1,
    \dots, N_\text{b}$, these equations are iterated from $t = 1, \dots,
    t_\text{max}$, at which point we assign $\lambda_{k + 1} = \lambda_{k +
    1}^{(t_\text{max})}$ and $\gamma_{k + 1} = \gamma_{k
    +1}^{(t_\text{max})}$. The MSE given by AMP after the $k$-th mini-batch
    has been processed is given by ${\cal E}_k$. In the Bayes-optimal case,
    in particular, one can further show that the state evolution reduces to
    \begin{equation}
        \begin{aligned}
            \lambda_{k}^{(t)} &= \lambda_{k - 1} + \frac{\alpha_\text{b}}{\Delta + {\cal E}_{k}^{(t)}}, \quad \quad 
            {\cal E}_{k}^{(t + 1)} &= \mathbb{E}_{x, z} \Big( \eta \big( \lambda_{k}^{(t)}, \lambda_{k}^{(t)} x + \sqrt{\lambda_{k}^{(t)}} z \big) - x \Big)^2. \,
        \end{aligned}
        \label{eq:minise_nish}
    \end{equation}
\end{theorem}

\begin{proof}
    We apply the proof of state evolution for AMP in
    \cite{bayati_dynamics_2011} to each mini-batch step, each with its own
    denoiser function $\eta(.)$ Each step is an instance of AMP with a new,
    independent matrix, and an effective denoiser given by
    \eqref{eq:miniamp}. Using \eqref{bayati}, the statistics of the
    denoisers are known, and the application of the standard AMP state
    evolution leads to \eqref{eq:minise}. The Bayes-optimal case
    \eqref{eq:minise_nish} then follows by induction, as in Sec. V.A.2 of
    \cite{kabashima_phase_2014}.
 \end{proof}
 
Note that the above Theorem holds for any value of $t_{\rm max}$. Hence,
even a stochastic version of the Mini-AMP algorithm, where for every
mini-batch one only performs a few iterations without waiting for
convergence in order to further speed up the algorithm, is analyzable using
the above state evolution.

\subsection{Optimal MMSE and mutual information under mini-batch setting}

\begin{theorem}[Mutual information for each mini-batch]\label{thm2} 
    For a random matrix $\Phi$, where each element is taken i.i.d. from
    ${\cal {N}}(0,1/N)$, in the Bayes-optimal setting, assume one has been
    given, after $k-1$ mini-batches, a noisy version ${\bm r}$ of unknown
    signal ${\bm x}$ with i.i.d noise ${\cal N}(0,\lambda^{-1})$. Given a
    new mini-batch with $\alpha_b\!=\!O(1)$, the mutual information per
    variable between the couple $({\bm r},{\bm y})$ and the unknown ${\bm
    x}$ is asymptotically given by $i={\rm min}~{i}^{b}_\text{RS} ({\cal
    E}_k)$ where, defining $\Sigma_b^{-2}(\lambda,{\cal E}_k) \equiv
    \lambda+\frac{\alpha_b}{\Delta+{\cal E}_k}$,
    \begin{equation}
        {i}^{b}_{\rm RS} ({\cal E}_k) = \frac{\alpha_b}{2} \left[ \frac{{\cal
        E}_k}{\Delta + {\cal E}_k} + \log{\left(1+ \frac{{\cal
        E}_k}{\Delta}\right)} \right] - \mathbb{E}_{x, z}  \log
        \mathbb{E}_{\tilde x} e^{-\frac{\left({\tilde x}-\left({x} + {z}
        \Sigma_b(\lambda,{\cal E}_k)\right)\right)^2}{2 \Sigma_b^2(\lambda,{\cal
        E}_k)}}-\frac {1+\alpha_b}2.
        \label{eq:miniphi}
    \end{equation}
\end{theorem}

\begin{proof}
    The proof is a slight generalization of the Guerra construction in
    \cite{barbier2017mutual_last}. Using properties of the Shannon entropy,
    the mutual information can be written as
    \begin{equation}
        I(Y,R;X) = H(Y,R)- H(Y|X) - H(R|X) = {\mathbb{E}}_{\bm y,{\bm r}} \log
        {\mathbb{E}}_{\bm x} {e^{-\frac{\|{\bm y}-\Phi {\bm x}
        \|^2_2}{2\Delta}}} {e^{-\frac{\|{\bm r}-{\bm x} \|^2_2}{2\lambda^{-1}}}}
        - \frac {1+\alpha_b}2 \nonumber.
    \end{equation}
    Computation of this expectation is simplified by noticing that it
    appears as equation (41) in the Guerra construction of
    \cite{barbier2017mutual_last}, where it was used as a proof method for
    the offline result by interpolating from a pure noisy Gaussian channel
    (at "time" $t\!=\!0$) to the actual linear channel (at "time"
    $t\!=\!1$). Authors of \cite{barbier2017mutual_last} denoted
    $\lambda(t)$ as the variance of the Gaussian channel and $\gamma(t)$ as
    the variance of the linear channel. Our computation corresponds instead
    to a "time" $0 \le \tau \le 1$ where both channels are used. Using
    Sec.~V of \cite{barbier2017mutual_last}, with the change of notation
    $\gamma(\tau)\!\to\!\Delta$ and $\lambda(\tau)\!\to\!\lambda$, we reach
    \begin{equation}
        \lim_{N \to \infty} \frac{I(Y,R;X)}N = {i}^{b}_\text{RS} ({\cal E}_k) -
        \int_0^{\tau} R_{{\cal E}_k}(t) ~~{\rm d}t + O(1),
    \end{equation}
    where $0 \le {\tau} \le 1$ and $R_{{\cal E}_b}(t)$ a non-negative
    function called the reminder. The validity of the mutual information
    formula in the offline situation
    \cite{barbier2017mutual_last,reeves_replica-symmetric_2016} implies that
    the integral of the reminder in $[0,1]$ is zero when ${\cal E}^* =
    \rm{argmin} \,~ {i}^{b}_\text{RS} ({\cal E}_b) $. Since $R(t)$ is
    non-negative, this implies that it is zero almost everywhere, thus
    $\int_0^{\tau} R_{{\cal E}^*_b}(t)~~{\rm d}t=0$. 
 \end{proof}

\begin{theorem}[MMSE for each mini-batch]\label{thm3}
    With the same hypothesis of Theorem \ref{thm2}, the MMSE when one has access
    to a noisy estimate with i.i.d. noise ${\cal N}(0,\lambda^{-1})$ and the
    data from the mini-batch at step $k$, is
    \begin{equation}
        \rm{MMSE} = \rm{argmin} ~ i_{\rm RS}^b({\cal E}_{k})\,.
    \end{equation}
\end{theorem}

\begin{proof}
    The proof follows again directly from generic results on the Guerra
    interpolation in \cite{barbier2017mutual_last} and the so-called I-MMSE
    formula ${\rm d}i(\Delta)/{\rm d}\Delta^{-1}=\alpha_b y_{\rm MMSE} /2$
    and y-MMSE formula $y_{\rm MMSE}={\rm MMSE}/(1+\Delta \, {\rm MMSE})$
    linking the mutual information and the MMSE.
\end{proof}

One can show through explicit computation that the extrema of $i^{b}_{\rm
RS}({\cal E})$ correspond -- just as in the offline case -- to the fixed
points of the state evolution.

Using these results, we can analyze, both algorithmically and information
theoretically, the mini-batch program \eqref{program}. Indeed, the new
information on the parameters $\bm x$ passed from mini-batch $k-1$ to $k$
contained in \eqref{eff-prior} is simply a (Gaussian) noisy version of $\bm
x$ with inverse variance $\lambda_{k-1}$. This is true for the AMP estimate
(see \eqref{eff-prior}) and, in the large $N$ limit, for the exact
marginalized posterior distribution as well (see e.g.
\cite{talagrand2003spin,lelarge2016fundamental,miolane2017fundamental}). The
optimal MSE at each mini-batch is thus given by the recursive application of
Theorem \ref{thm3}, where at each mini-batch $k=1,\ldots, B$ we minimize
\eqref{eq:miniphi} using $\lambda_{k-1}=\lambda_{k-2}+
{\alpha_b}/{(\Delta+\rm{MMSE}_{k-1})}$.

Now we can compare the MSE reached by the Mini-AMP algorithms to the MMSE.
If, for each mini-batch, the MMSE is reached by the state evolution of the
Mini-AMP algorithm starting from the previously reached MSE, then we have
the remarkable result that Mini-AMP performs a {\it Bayes-optimal} and {\it
computationally efficient} implementation of the mini-batch program
\eqref{program}. Otherwise, the Mini-AMP is suboptimal. It remains an open
question whether in that case any polynomial algorithm can improve upon the
MSE reached by Mini-AMP. 

All our results can be directly generalized to the case of AMP for matrix or
tensor factorization, as derived and proven in
\cite{lesieur2017constrained,miolane2017fundamental,lesieur2017statistical}.
They can also be adapted to the case of GLMs with non-linear output
channels. However, in this setting, the formula for the mutual information
has not been yet proven rigorously.

\section{Performance and phase transitions on GLMs}
\label{sec:glms}
\subsection{Optimality \& efficiency trade-offs with Mini-AMP}
\begin{figure}[ht]
	\centering
    \includegraphics[height=5.2cm]{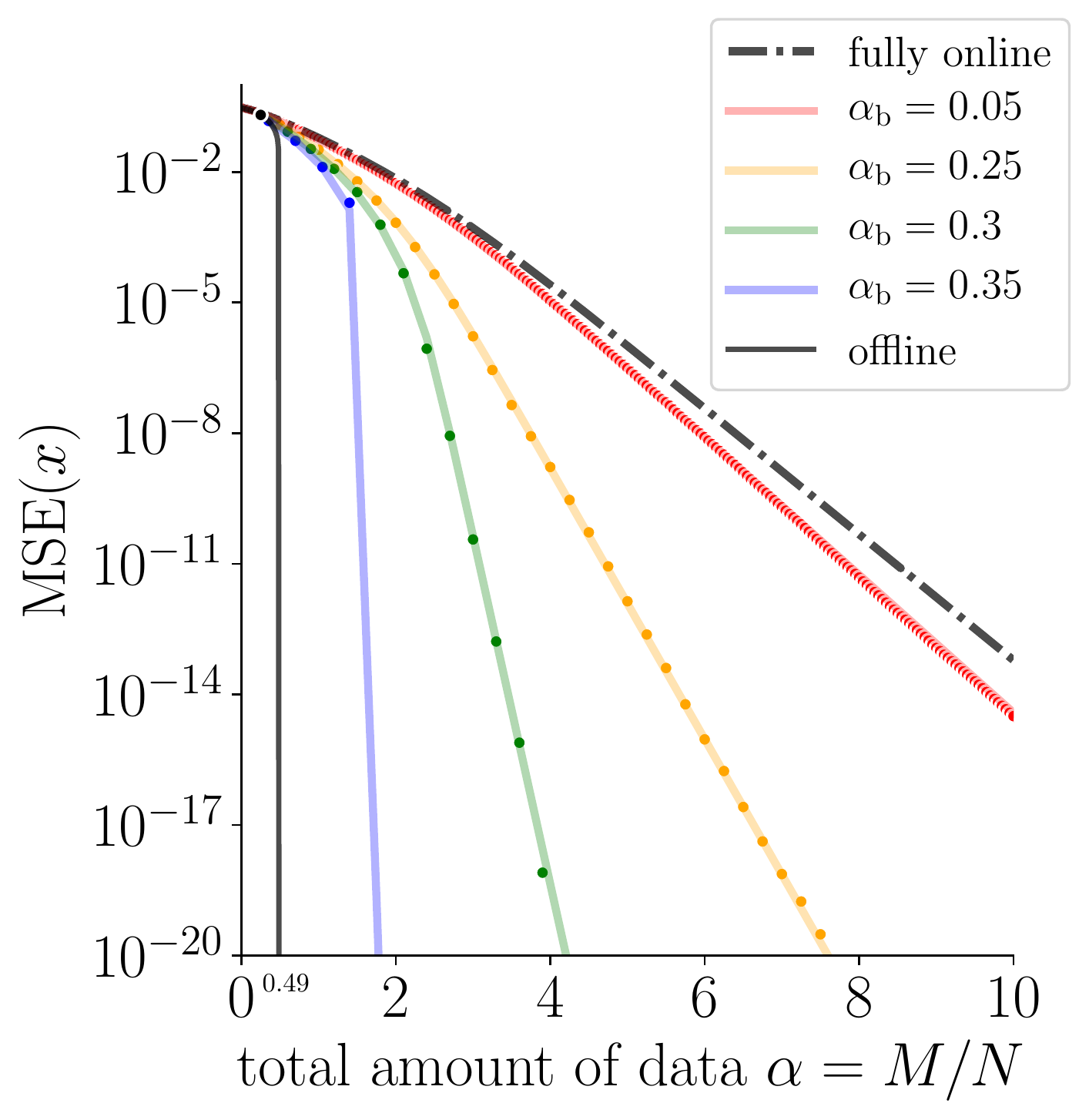}
    \includegraphics[height=5.2cm]{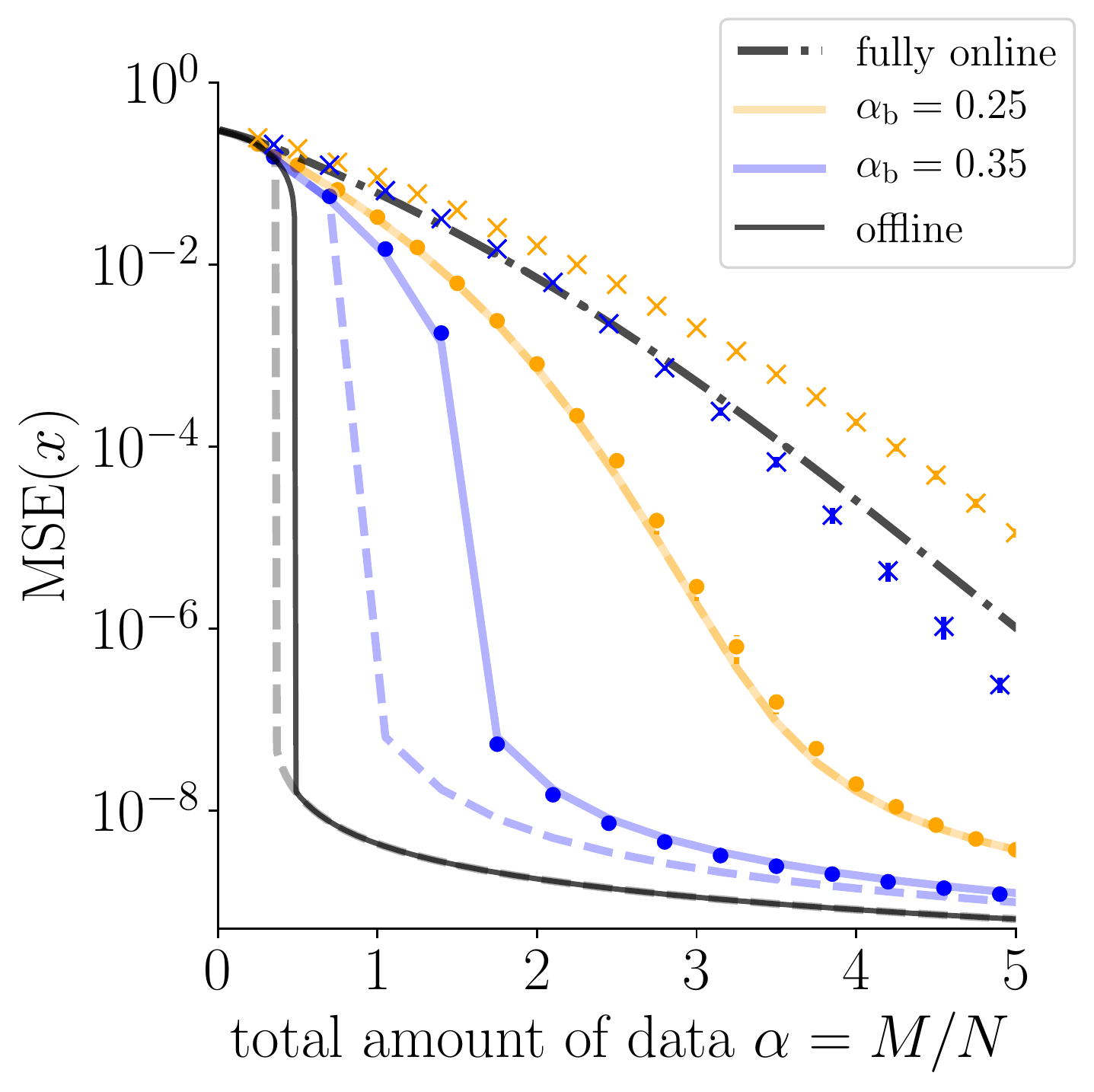}
    \includegraphics[height=5.2cm]{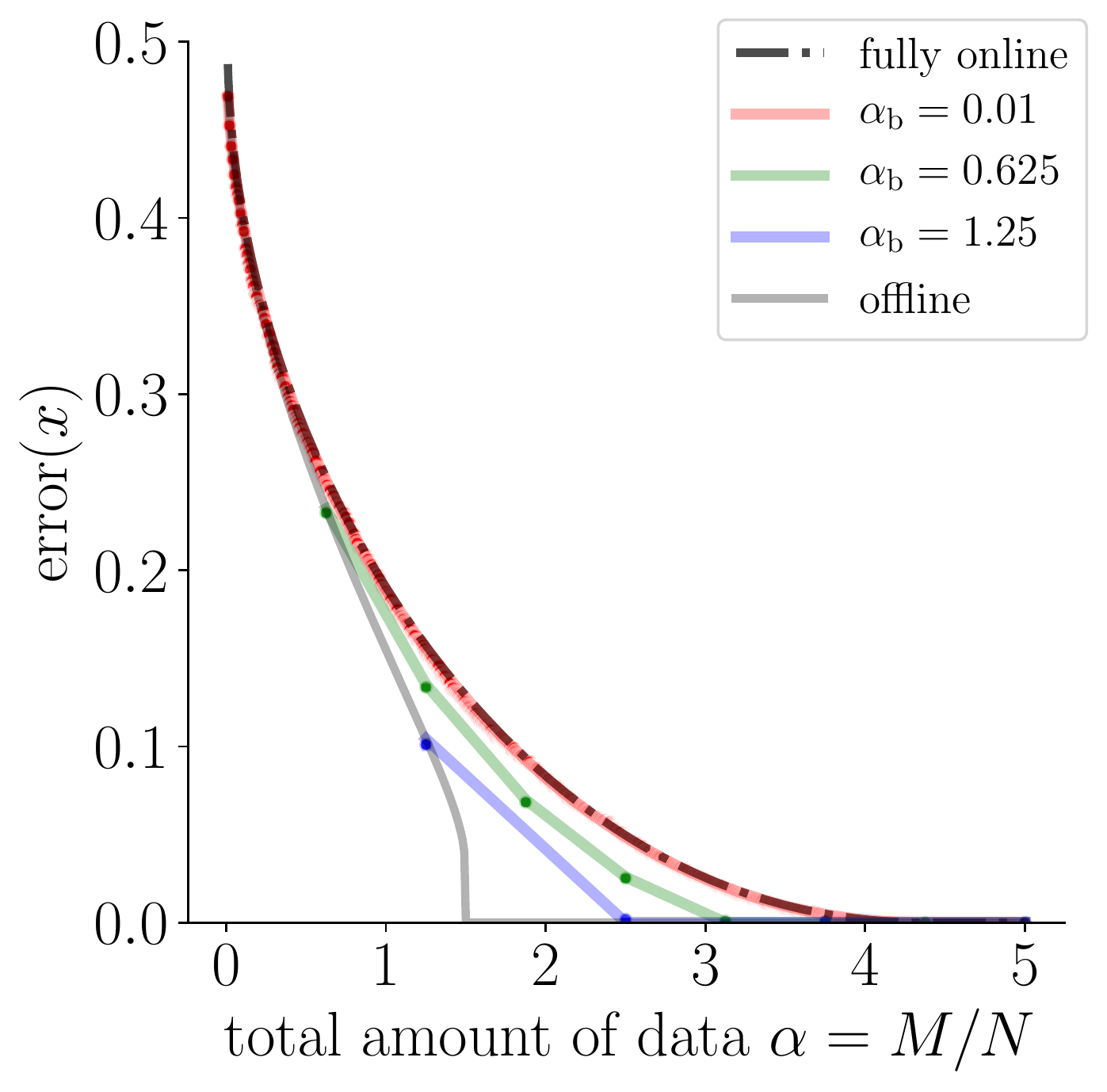}
    \caption{\small
        Accuracy of Mini-AMP inference as a function of the total amount of
        presented data.
        \emph{Left:} SLR with $\Delta = 0$ and sparsity $\rho = 0.3$.
        \emph{Center:} SLR with $\Delta = 10^{-8}$ and $\rho = 0.3$.
        \emph{Right:} Perceptron learning with Rademacher ($\pm 1$)
        parameters/synapses.
        For different mini-batch sizes (\emph{colors}), we show both the
        state evolution predictions for Mini-AMP (\emph{solid lines}), the
        predicted MMSE (\emph{dashed lines}, only center pannel), and
        empirical experiments for Mini-AMP ($\bullet$) and streaming VB
        \cite{broderick_streaming_2013,krzakala_variational_2014} ($\times$)
        averaged over 10 realizations of size $N = 2000$. We also show the
        results for both fully offline (\emph{solid black line}) and fully
        online (\emph{dash-dot black line}) algorithms. Even for moderate
        $N$ the state evolution is found to almost perfectly describe
        Mini-AMP’s behavior. For the parameters of the center plot we
        observed that for $\alpha_b\!\lesssim\!0.33$ Mini-AMP is
        asymptotically optimal.
    }
    \label{fig:tradeoff}
\end{figure}

We now illustrate the above results on some examples. In Figure
\ref{fig:tradeoff}, we consider the SLR model and the perceptron with binary
$\pm1$ parameters, both with random matrices $\Phi \in \mathbb{R}^{M \times
N}$, $\Phi_{\mu i} \sim \mathcal{N} (0, 1/N)$. Our analysis quantifies the
loss coming from using mini-batches with respect to a fully offline
implementation. In the limit of small mini-batch $\alpha_\text{b} \to 0$, we
recover the results of the ADF algorithm which performs fully online
learning, processing one sample at a time
\cite{solla_optimal_1998,rossi_bayesian_2016}. This suggests that the state
evolution accurately describes the behavior of Mini-AMP beyond the
theoretical assumption of $\alpha_b\!=\!O(1)$, even for mini-batches as
small as a single sample. 

The effect of the mini-batch sizes varies greatly with the problem. For the
perceptron with $\pm1$ weights, a zero error is eventually obtained after a
sufficient number of mini-batches have been processed. Moreover, the
dependence on the mini-batch size is mild: while the offline scheme achieves
zero error at $\alpha \approx 1.5$
\cite{gardner_three_1989,gyorgyi_first-order_1990}, the fully online does it
at $\alpha \approx 4.4$ \cite{solla_optimal_1998}, that is, going from
offline to a fully online scheme costs only about three times more data
points. The behavior of the Mini-AMP for SLR shows instead rather drastic
changes with the mini-batch size. The MSE decays smoothly when the
mini-batch size is small. However, as we increase it, a sudden decay occurs
after a few mini-batches have been processed. For the noiseless case
($\Delta = 0$), the study of the state evolution shows that the asymptotic
(in $\alpha$) MSE is given by
\begin{equation}
    \operatorname{MSE}_x (\alpha) \sim e^{-\frac{1}{\alpha_\text{b}} \log (1
    - \frac{\alpha_\text{b}}{\rho} ) \, \alpha},
    \label{eq:asymptotic}
\end{equation}
if $\alpha_\text{b} \leq \rho$, and by 0 otherwise. These results provide a
basis for an optimal choice of mini-batch size. Given the drastic change in
behavior past a certain mini-batch size, one concludes that small
investments in memory might be worthwhile, since they can lead to large
gains in performance.

Finally, we have compared the Mini-AMP scheme with the streaming VB approach
\cite{broderick_streaming_2013} using the mean-field algorithm described in
\cite{krzakala_variational_2014} for SLR. While the mean-field approach is
found to give results comparable to AMP in the offline case in
\cite{krzakala_variational_2014}, we see here that the results are
considerably worse in the streaming problem. In fact, as shown in Figure 1
(center), mean-field can give worse performance than the fully online ADF,
even when processing rather large mini-batches.

\subsection{Phase transitions}

\begin{figure}[ht]
	\centering
    \includegraphics[height=6cm]{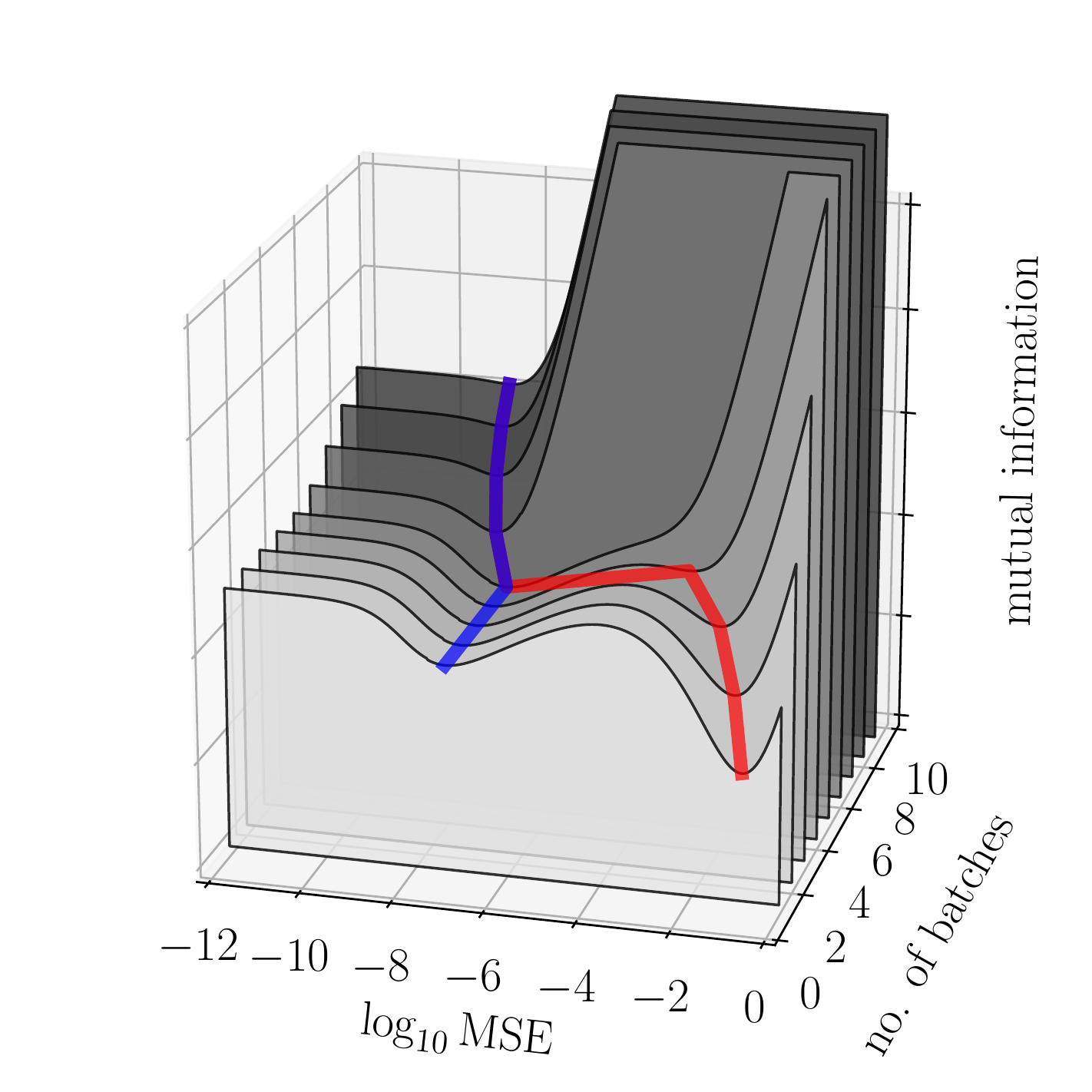}
    \qquad
    \includegraphics[height=6cm]{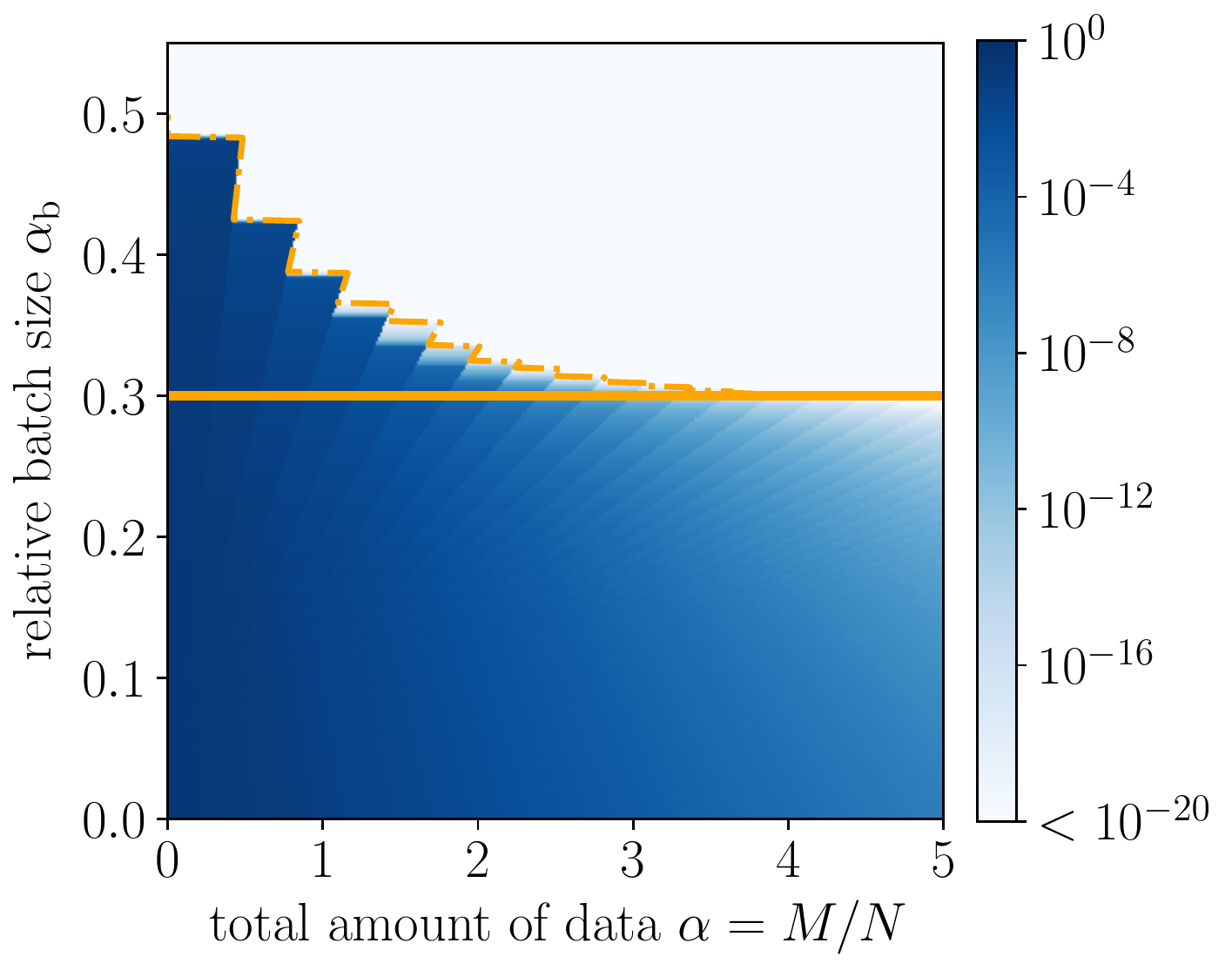}
    \caption{
        \small Phase transitions in streaming SLR.
        \emph{Left}: evolution of the mutual information in the streaming
        SLR problem as each mini-batch is processed. Parameters are set to
        $\rho = 0.3$, $\Delta = 10^{-8}$ and $\alpha_\text{b} = 0.35$.
        \emph{Right}: MSE of Mini-AMP for different mini-batch sizes.
        Mini-AMP achieves the MMSE for $\alpha_b<\rho$. For $\alpha_b>\rho$
        the MMSE is zero after processing a single batch, while for batch
        sizes between $\rho < \alpha_b < 0.49$ the Mini-AMP is suboptimal
        unless a sufficient number of mini-batches is processed.
    }
    \label{fig:phasetrans}
\end{figure}
It turns out that, just as for the offline setting, there are phase
transitions appearing for mini-batch learning, in terms of the learning
error that is achievable information-theoretically (MMSE) {\it or}
computationally efficiently (by Mini-AMP). These can be understood by an
analysis of the function $i^b_{\rm RS}$, since the minimum of $i^b_{\rm
RS}$ gives the MMSE, and since AMP is effectively trying to minimize
$i^b_{\rm RS}$ starting from the MSE reached at the previous mini-batch
steps.

Let us illustrate the reason behind the sharp phenomenon in the behavior of
AMP in Fig.\ref{fig:tradeoff}. We show, in Fig. \ref{fig:phasetrans}
(left), an example of the function $i^b_{\rm RS}({\cal E})$ for the
streaming SLR problem as a function of the MSE ${\cal E}$ as each mini-batch
is being processed. Initially, it presents a ``good'' and a ``bad''
minimum, at small and large MSEs respectively. In the very first batch, AMP
reaches the bad minimum. As more batches are processed, the good minimum
becomes global, but AMP is yet not able to reach it, and keeps returning the
bad one instead. This indicates a computational phase transition, and we
expect that other algorithms will, as AMP, fail to deliver the MMSE in
polynomial time when this happens. Eventually, the good minimum becomes
unique, at which point AMP is able to reach it, thus yielding the sudden
decay observed in Figure \ref{fig:tradeoff}.

Consider now the Bayes-optimal "streaming-MMSE" given by the global minimum
of the mutual information at each step, regardless of whether AMP achieves
it. In the offline noiseless case, the MMSE is achieved by AMP only if the
processed batch has size $\alpha \geq \alpha_{\rm offline} $ or $\alpha\le
\rho$~\cite{krzakala_probabilistic_2012}. In the streaming case, we also
observe that Mini-AMP reaches the streaming-MMSE if the mini-batch size is
sufficiently small {\it or} sufficiently large. In Figure
\ref{fig:phasetrans} (right) we compare MMSE to the MSE reached by Mini-AMP,
with a region between the full and dashed line where the algorithm is
sub-optimal.

\section{Mini-AMP for matrix factorization problems}
\label{sec:lowrank}
\begin{figure}[ht]
    \centering
    \includegraphics[height=7cm]{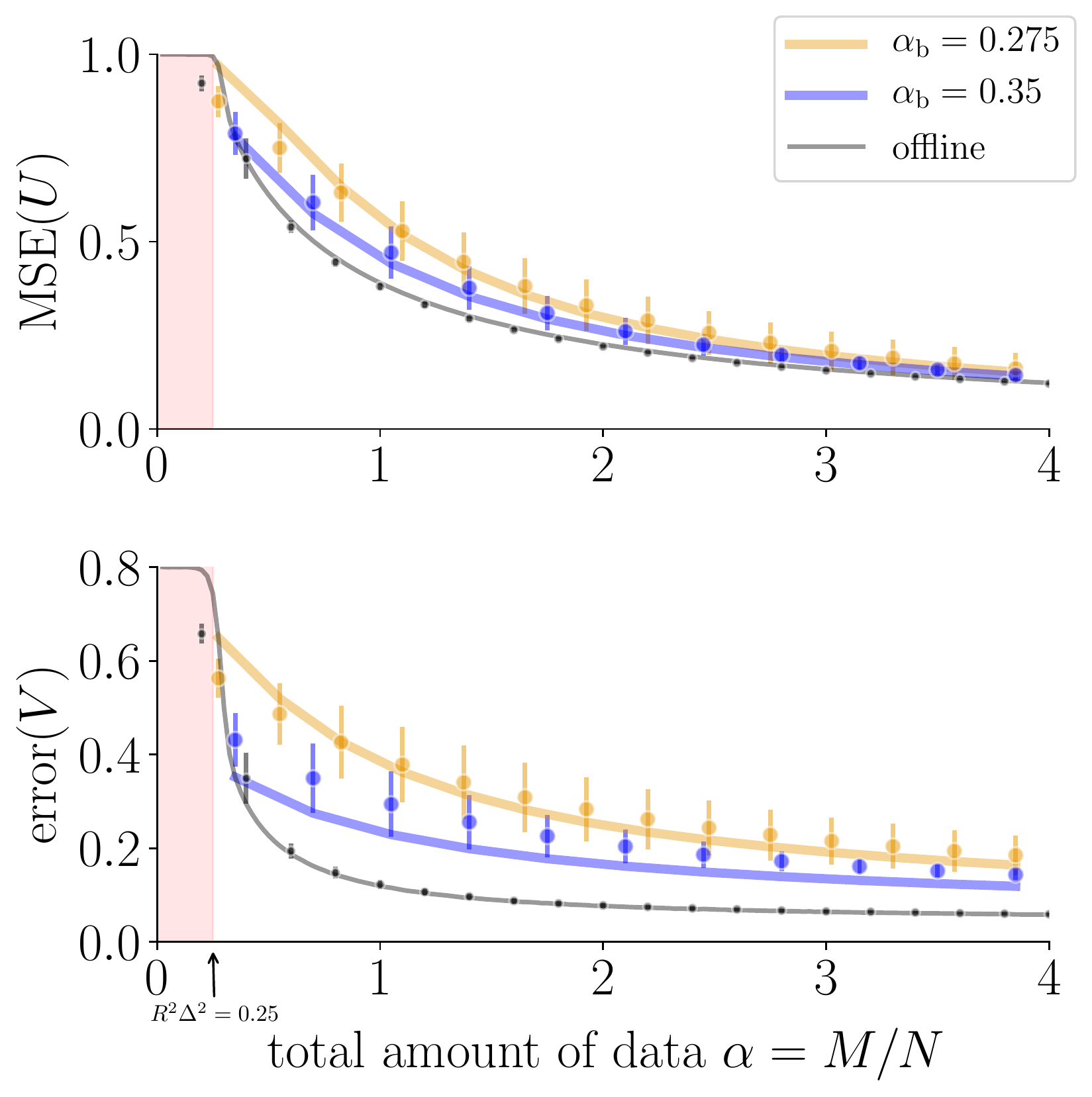}
    \kern3em
    \includegraphics[height=7cm]{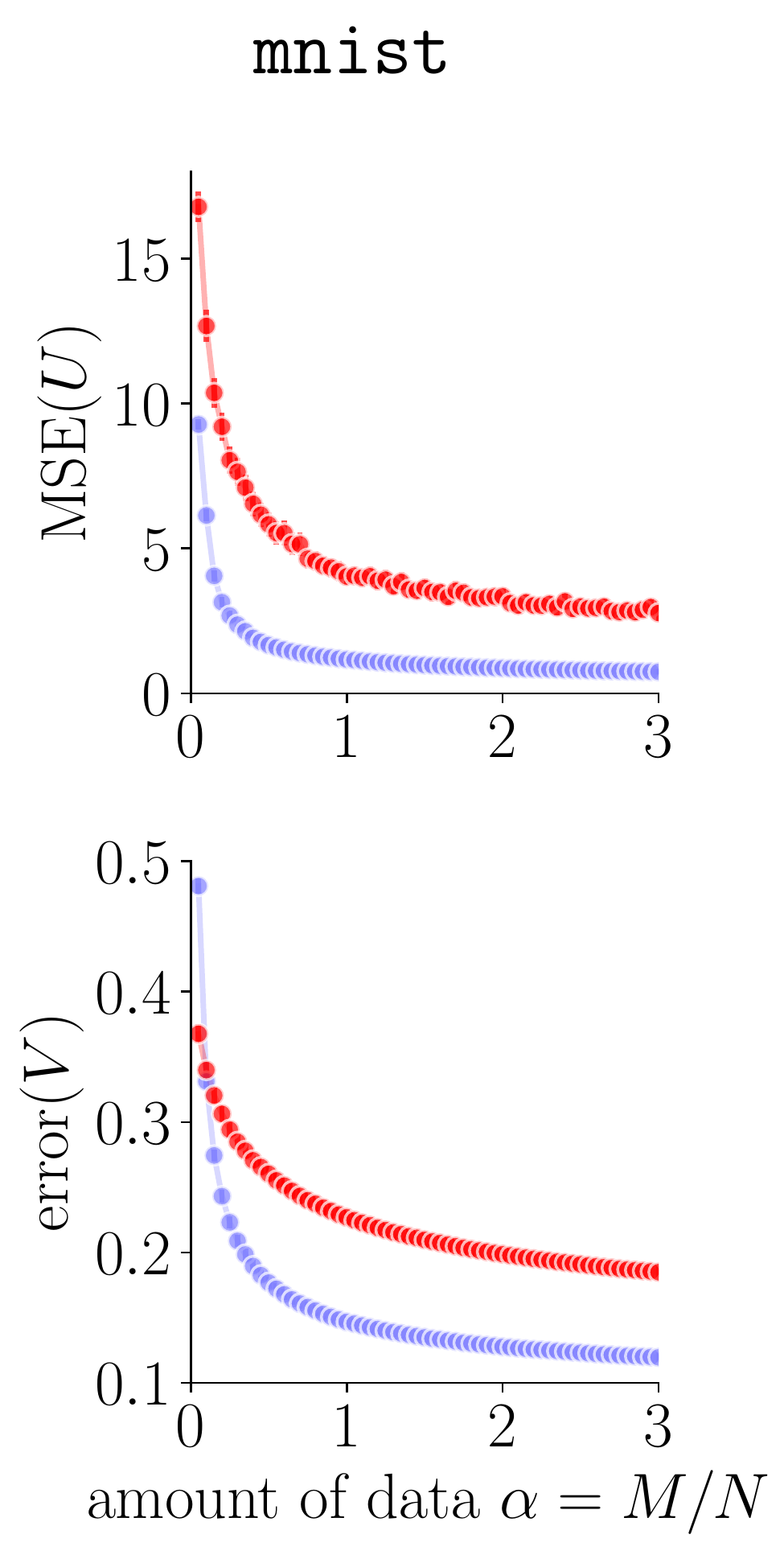}\!
    \includegraphics[height=7cm]{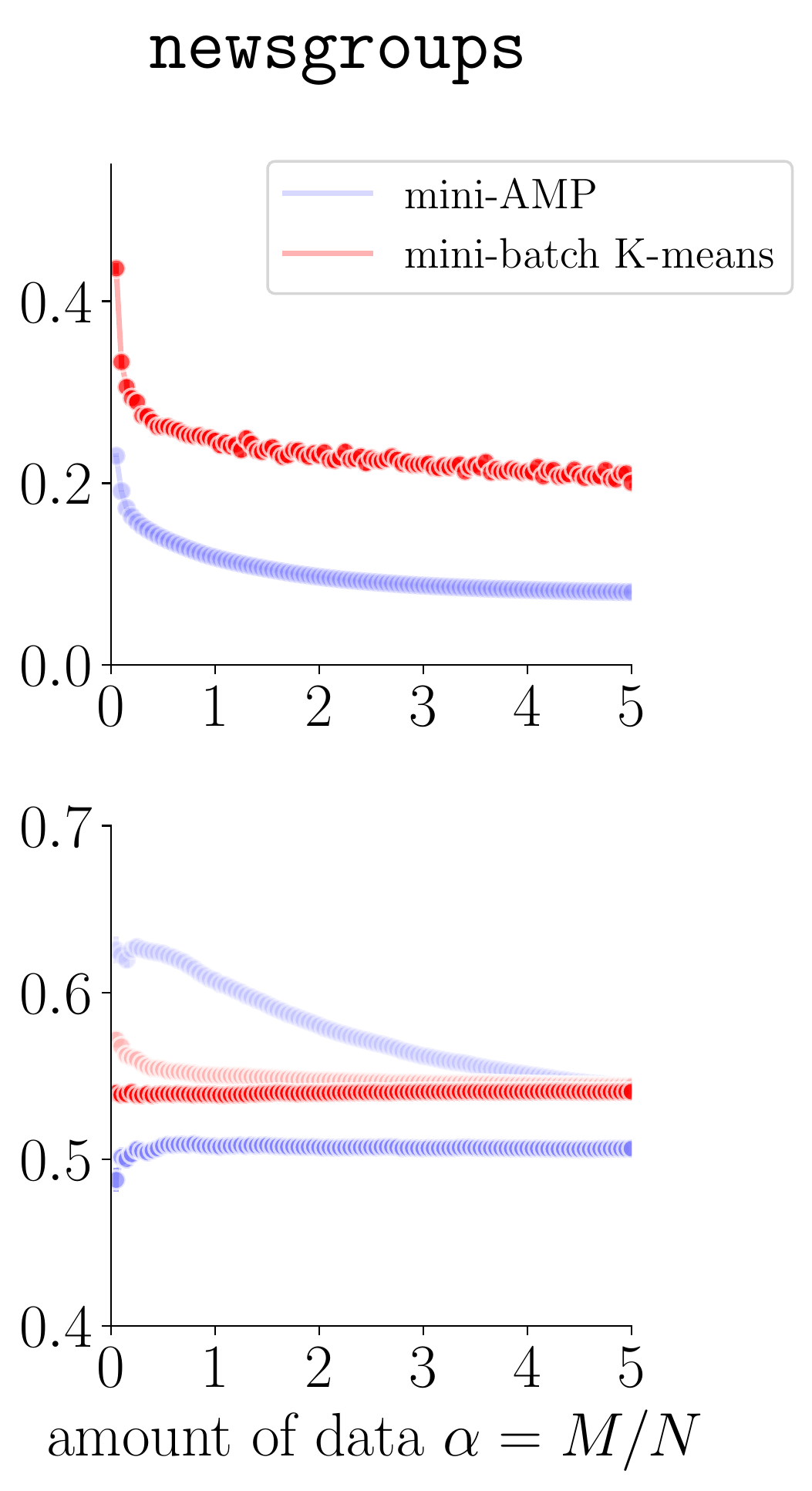}
    \caption{\small Clustering with Mini-AMP on synthetic Gaussian mixture
    data (left) and real-world data (right). 
    \emph{Left}: mean-squared error in $U$ (centroids) and 0-1 loss in $V$
    (labels) using different batch sizes. Solid lines give state evolution
    and symbols give averages over 100 instances of size $N = 1000$. A
    transition at $\alpha_\text{c} = R^2 \Delta^2 = 0.25$ prevents Mini-AMP
    from giving non-zero overlap when $\alpha_\text{b} <
    \alpha_\text{c}$. Parameters are set to $R = 5$, $\Delta = 0.1$. 
    \emph{Right}: clustering on MNIST and the 20 newsgroups dataset using
    Mini-AMP for model (\ref{eq:lowrank}), with the prior on $U$ replaced with a
    non-negative Gaussian of mean zero and variance $0.1$, and the noise
    variance $\Delta$ estimated from the data. On MNIST, digits of size $N
    = 784$ were clustered on $K = 3$ classes (0, 1 and 2), whereas for the
    20 newsgroups dataset, frequency statistics of $N = 1000$ words and $K =
    3$ top-level hierarchies (\texttt{comp}, \texttt{rec}, and \texttt{sci})
    were used. Batch sizes were set so that $\alpha_b = 0.05$. Blue/red
    circles give the cumulative performance of Mini-AMP and mini-batch
    K-means \cite{sculley_web-scale_2010}, respectively, averaged over a 100
    different orders of presentation. For the newsgroups dataset, a 2nd pass
    over the data was performed so that all labels could be recomputed with
    accurate estimation of the centroids; results of 1st and 2nd passes are
    shown in light/dark lines respectively.}
\label{fig:clust_msealpha}
\end{figure}

We now consider the case of low-rank matrix factorization, and in particular
clustering using the Gaussian mixture model (GMM) with $R$ clusters. For
such problems, the generative model reads
\begin{equation}
    P (U, V, Y) = \prod_{ij} \mathcal{N} (Y_{ij} ; \bm{U}_i \cdot \bm{V}_j,
    \Delta) \prod_{i = 1}^N \mathcal{N} (\bm{U}_i; \bm{0}, \mathbb{I}_R)
    \prod_{j = 1}^M \frac{1}{R} \sum_{k = 1}^R \delta (\bm{V}_j -
    \mathbbm{e}_k),
    \label{eq:lowrank}
\end{equation}
where $\bm{U}_i$ and $\bm{V}_j$ give the $i$-th row of $U$ and $j$-th row of
$V$ respectively. Each of the $R$ columns of $U$ describe the mean of a
$N$-variate i.i.d. Gaussian, and $V$ has the role of picking one of these
Gaussians. Finally, each column of $Y$ is given by the chosen column of $U$
plus Gaussian noise. In clustering, these are the data points, and the
objective is to figure out the position of the centroids as well as the
label assignment, given by the columns of $U$ and the rows of $V$
respectively. In the streaming setting the columns of the matrix $Y$ are
arriving in mini-batches. The offline AMP algorithm, its state evolution,
and corresponding proofs are known for matrix factorization from
\cite{rangan_iterative_2012,matsushita_low-rank_2013,lesieur_phase_2016,lesieur2017constrained,barbier_mutual_2016-1,miolane2017fundamental}.
The Mini-AMP is obtained by adjusting the update of the estimators using
\eqref{eq:miniamp}.

In GMM clustering with prior on U having zero mean, there is an interesting
"undetectability" phase transition for $R \le 4$. If the number of samples
is such that $\alpha = M/N < \alpha_c = R^2 \Delta^2$, then the Bayes
optimal posterior asymptotically does not contain any information about the
ground truth parameters \cite{lesieur_phase_2016}. This transition survives
even when $R>4$, in the sense that AMP and other tractable algorithms are
unable to find any information on the ground truth parameters. 

In the streaming problem, this undetectability implies that for mini-batches
of relative size $\alpha_\text{b} < \alpha_\text{c}$, Mini-AMP does not
improve the error of the random estimator, no matter the number of
mini-batches presented. In particular, the fully online algorithm does not
provide {\it any} useful output in this scenario. If $\alpha_\text{b} >
\alpha_\text{c}$, on the other hand, an accurate reconstruction of the
unknown values becomes possible. We illustrate the MSE as a function of the
mini-batch size in Figure \ref{fig:clust_msealpha}. 

While we have presented Mini-AMP as a means for a theoretical analysis, it
can be applied to real data, performing concrete learning tasks. To
illustrate its efficacy, we have considered the classical problem of
unsupervised clustering using the GMM. In Figure \ref{fig:clust_msealpha},
Mini-AMP is shown to obtain better performance for real data clustering than
mini-batch K-means, a state-of-the-art algorithm for streaming clustering
\cite{sculley_web-scale_2010}.

\section{Conclusion}
\label{sec:conclusion}
Let us conclude by stating that the Mini-AMP algorithm can be applied to any
problem for which the streaming can be defined and for which offline AMP
exists. Therefore, we expect that this novel development will improve the
usefulness of AMP algorithms in more practical situations.

\section*{Acknowledgments}
This work has been supported by the ERC under the European Union's FP7 Grant
Agreement 307087-SPARCS. AM thanks Paulo V. Rossi and Thibault Lesieur for
insightful discussions.

\bibliographystyle{unsrtnat}
\bibliography{refs}

\clearpage

\appendix
\section{AMP equations for different classes of models}
\label{sec:appendix_amp}
We present here the AMP equations for different models. As before, adapting them 
to the streaming setting is done by introducing $\bm{\Lambda}$, $\bm{\Theta}$ 
variables and replacing the $\eta (A, B)$ function with (\ref{eq:miniamp}).

\subsection{Generalized linear models}

Denote by $\bm{y} \in \mathbb{R}^{M}$ the response variable, by $\Phi \in
\mathbb{R}^{M \times N}$ the design matrix, and by $\bm{x} \in \mathbb{R}^N$
the parameter vector that we want to estimate. Then our generative model reads
\begin{equation}
    P(\bm{y}, \bm{x} | \Phi) = \prod_{\mu = 1}^{M} P(y_\mu | z_\mu \equiv
        \bm{\Phi}_{\mu} \cdot \bm{x}) \, \prod_{i = 1}^N P_X(x_i).
\end{equation}
The GAMP algorithm provides the following approximation to the marginals of 
$\bm{x}$
\begin{equation}
q_x (x_i | A_i, B_i) = \frac{1}{Z_x (A_i, B_i)} \, P_X (x_i) \, e^{-\frac{1}{2} A_i x_i^2 + B_i x_i},
\end{equation}
and, to the marginals of $\bm{z}$
\begin{equation}
q_z (z_\mu | y_\mu, \omega_\mu, V_\mu) = \frac{1}{Z_z (y_\mu, \omega_\mu, V_\mu)}
\, P(y_\mu | z_\mu) \, \frac{e^{-\frac{(z_\mu - \omega_\mu)}{V_\mu}}}{\sqrt{2 \pi V_\mu}}.
\end{equation}
The parameters $\bm{A}$, $\bm{B}$, $\bm{\omega}$ and $\bm{V}$ are determined by 
iterating the GAMP equations. We denote the mean and variance of $q_x (A, B)$ by 
$\eta (A, B) = \frac{\partial}{\partial B} \log Z_x (A, B)$ and $\eta' (A, B) = 
\frac{\partial \eta}{\partial B} (A, B)$ respectively; moreover, we define 
$g_\text{out} (y, \omega, V) = \frac{\partial}{\partial \omega} \log Z_z (y, 
\omega, V)$. The GAMP equations then read 
\cite{rangan_generalized_2011,zdeborova_statistical_2016}
\begin{equation}
    \begin{aligned}
        &\bm{\omega}^{(t)} = \Phi \hat{\bm{x}}^{(t)} - \bm{V}^{(t)} \circ 
            \bm{g}^{(t - 1)}, &\qquad
        &\bm{V}^{(t)} = (\Phi \circ \Phi) \, \hat{\bm{\sigma}}^{(t)}, \\
        &g_\mu^{(t)} = g_\text{out} (y_\mu, \omega_\mu^{(t)}, V_\mu^{(t)}) \;
            \forall \mu, &\qquad
        &\partial_\omega g_\mu^{(t)} = \partial_\omega g_\text{out}
            (y_\mu, \omega_\mu^{(t)}, V_\mu^{(t)}) \; \forall \mu, \\
        &\bm{B}^{(t)} = \Phi^T \bm{g}^{(t)} + \bm{A}^{(t)} \circ 
            \hat{\bm{x}}^{(t)}, &\qquad
        &\bm{A}^{(t)} = -(\Phi \circ \Phi)^T \, \partial_\omega \bm{g}^{(t)}, \\
        &\hat{x}_i^{(t + 1)} = \eta(A_i^{(t)}, B_i^{(t)}) \; \forall i, &\qquad
        &\hat{\sigma}_i^{(t + 1)} = \eta' (A_i^{(t)}, B_i^{(t)}) 
            \; \forall i.
    \end{aligned}
	\label{eq:gamp}
\end{equation}
Note that, for a Gaussian likelihood $P (y | z) = \mathcal{N} (y; z, \Delta)$, 
$g_\text{out} (y, \omega, V) = \frac{y - \omega}{\Delta + V}$. Then, by defining 
$\bm{\tilde{z}}^{(t)} = \bm{y} - \bm{\omega}^{(t)}$ and replacing $\bm{V}^{(t)}$ 
and $\bm{A}^{(t)}$ by its averages, we get via the central limit theorem
\begin{equation}
V^{(t)} \approx \frac{1}{N} \sum_{i = 1}^N \sigma^{(t)} (A^{(t)}, B_i^{(t)}), \qquad
A^{(t)} \approx \frac{\alpha}{\Delta + V^{(t)}},
\end{equation}
where we have assumed the $\Phi_{\mu i}$ are i.i.d. and have zero mean and
variance $1/N$. From these (\ref{eq:amp1})-(\ref{eq:amp3}) follow through.

\subsubsection{Variational Bayes}

We compare AMP equations to the Variational Bayes (VB) ones, which we
use with the Streaming Variational Bayes scheme. For simplicity we restrict
ourselves to the Gaussian case. As usual, VB is derived by determining
the $q_i (x_i)$ which minimize
\begin{equation}
    KL[{\textstyle \prod_i} q_i (x_i) \| P (\bm{x} | \Phi, \bm{y})] =
    -\mathbb{E}_{\{q_i\}} \log P (\bm{y} | \Phi, \bm{x}) + \sum_{i = 1}^N
    \operatorname{KL} [q_i (x_i) \| P_X (x_i)] - \log P (\bm{y} | \Phi)
\end{equation}
If done by means of a fixed-point iteration, this minimization leads to
\cite{krzakala_variational_2014}
\begin{equation}
	\hat{x}_i^{(t + 1)} = \eta\bigg(
		\frac{1}{\Delta} \sum_\mu \Phi_{\mu i}^2,
		\frac{1}{\Delta} \sum_\mu \Phi_{\mu i} (y_\mu - {\textstyle\sum_{j \neq
	i}} \Phi_{\mu j} \hat{x}_j^{(t)}) \bigg) \; \forall i,
\end{equation}
where $\eta$ is defined as before. A closer inspection shows that the
same equations are obtained by setting $\bm{V}^{(t)} = 0$ in (\ref{eq:gamp}).

As shown in \cite{krzakala_variational_2014}, this iteration leads to good results
when performed sequentially, \emph{if} the noise $\Delta$ is not fixed but learned.
We employ this same strategy in our experiments.

\subsubsection{Assumed Density Filtering}

The assumed density filtering (ADF) algorithm
\cite{opper_bayesian_1998,minka_expectation_2001} replaces the posterior
at each step $k$, $P(\bm{x} | \bm{\Phi}_k, y_k) \propto P(y_k | \bm{\Phi}_k,
\bm{x}) \, Q_{k - 1} (\bm{x})$, by the distribution $Q_k (\bm{x})$ that
minimizes
\begin{equation}
    \operatorname{KL} [P(\bm{x} | \bm{\Phi}_k, y_k) \| Q_k (\bm{x})].
    \label{eq:kl_adf}
\end{equation}
Note this is the \emph{direct} KL divergence, and not the
\emph{reverse} one $\operatorname{KL} [Q (\bm{x}) \| P(\bm{x} | \Phi,
\bm{y})]$ that is minimized in Variational Bayes. In particular
if $Q_k (\bm{x}) = \prod_{i = 1}^N q_k (x_i)$, then minimizing this KL
divergence leads to the following integral
\begin{equation}
    q_k (x_i) = q_{k - 1} (x_i) \int \bigg[\prod_{j \neq i} dx_j \, q_{k - 1}
    (x_j)\bigg] \, P(y_k | \bm{\Phi}_k, \bm{x}) \; \forall i,
    \label{eq:adf}
\end{equation}
which is tractable since we are processing a single sample $y_k$ at a time,
i.e. since the likelihood consists of a single factor. These are actually
the exact marginals of $P(\bm{x} | \bm{\Phi}_k, y_k)$, and also the equations
given by the belief propagation (BP) algorithm in the single sample limit.

The ADF equations for GLMs can be derived by using (\ref{eq:adf}) together
with the central limit theorem \cite{rossi_bayesian_2016}. Because
BP gives ADF in the single sample limit, AMP (which is based on BP) gives
the equations derived by \cite{rossi_bayesian_2016} when $M = 1$, if
one additionally neglects the correction term $-\bm{V}^{(t)}
\bm{g}^{(t - 1)}$ (analogously, if a single iteration is performed).

\begin{algorithm}[ht!]
    \caption{Assumed Density Filtering for GLMs \cite{rossi_bayesian_2016}}
    \algnewcommand\algorithmicto{\textbf{to}}
    \algrenewtext{For}[3]%
    {\algorithmicfor\ $#1 \gets #2$ \algorithmicto\ $#3$ \algorithmicdo}
    \begin{algorithmic}[1]
        \State initialize $\Lambda_{0, i} = 0$, $\Theta_{0, i} = 0 \; \forall i$
        \For{k}{1}{M}
            \State compute $\omega_k = \bm{\Phi}_k \cdot \hat{\bm{x}}_{k - 1}$, $V_k = (\bm{\Phi}_k \circ \bm{\Phi}_k) \cdot \hat{\bm{\sigma}}_{k - 1}$
            \State compute $g_k, \partial_\omega g_k$ following (\ref{eq:gamp})
            \State compute $\bm{A}_k$, $\bm{B}_k$ following (\ref{eq:gamp})
            \State compute $\bm{\hat{\sigma}}_{k}$, $\hat{\bm{x}}_{k}$ following (\ref{eq:miniamp})
            \State accumulate $\bm{\Lambda}_{k} \gets \bm{\Lambda}_{k - 1} + \bm{A}_k$
            \State accumulate $\bm{\Theta}_{k} \gets \bm{\Theta}_{k - 1} + \bm{B}_k$
        \EndFor
    \end{algorithmic}
\end{algorithm}

\subsection{Low-rank matrix factorization}

Denote by $Y \in \mathbb{R}^{N \times M}$ the matrix we want to factorize,
and by $U \in \mathbb{R}^{N \times R}$, $V \in \mathbb{R}^{M \times R}$
the matrices which product approximates $Y$. The generative model then reads
\begin{equation}
    P(Y, U, V) = \prod_{ij} P(Y_{ij} | W_{ij} \equiv \bm{U}_i \cdot \bm{V}_j) \,
    \prod_{i = 1}^N P_U(\bm{U}_i) \, \prod_{j = 1}^M P_V(\bm{V}_j),
\end{equation}
where $\bm{U}_i$ and $\bm{V}_j$ denote the $i$-th and $j$-th rows of $U$ and $V$
respectively. The algorithm provides the following approximation to the
marginal of $\bm{U}_i$
\begin{equation}
    q_U (\bm{U}_i | A_U, \bm{B}_{U,i}) = \frac{1}{Z_U (A_U,
    \bm{B}_{U, i})} \, P_U (\bm{U}_i) \, e^{-\frac{1}{2} \bm{U}_i^T A_U
    \bm{U}_i + \bm{B}_{U,i}^T \bm{U}_i}.
\end{equation}
and $q_V$ is analogously defined as the marginal of $\bm{V}_j$.  As in the
previous case, $A_U \in \mathbb{R}^{R \times R}$ and $\bm{B}_{U, i} \in
\mathbb{R}^R$ are to be determined by iterating a set of equations.  We
denote by $B_U$ the $N \times R$ matrix which rows are given by $\bm{B}_{U,
i}$, $i = 1, \dots, N$.
The functions $\eta_U (A, \bm{B}) = \mathbb{E}_{q_{U}} \, \bm{U} =
\nabla_{\bm{B}} \log Z_U (A, \bm{B})$ and $\eta'_U (A, \bm{B}) =
\nabla_{\bm{B}} \eta_U (A, \bm{B})$ give the mean and covariance of $q_U$,
and $\eta_V$ and $\eta'_V$, defined analogously, the mean and covariance of
$q_V$.

In order to write the AMP equations, we first introduce
\begin{equation}
    \begin{aligned}
    J_{ij} &= \frac{1}{\sqrt{N}} \frac{\partial \ln P (Y_{ij} | w = 0)}{\partial w}, \\
    \beta &= \frac{1}{N} \, \mathbb{E}_{P (y | w = 0)} \bigg[ \frac{\partial \ln P (y | w = 0)}{\partial w} \bigg]^2.
    \end{aligned}
\end{equation}
so as to define an effective Gaussian channel \cite{lesieur_mmse_2015}. The
equations to be iterated are then
\begin{equation}
    \begin{aligned}
        &B_U^{(t)} = J \hat{V}^{(t)} - \beta \Sigma_V^{(t)} \hat{U}^{(t - 1)}, &\qquad
        &A_U^{(t)} = \beta \, \hat{V}^{(t)} \hat{V}^{{(t)}^T}, \\
        &\hat{\bm{U}}_i^{(t)} = \eta (A_U^{(t)}, \bm{B}_{U, i}^{(t)}) \; \forall i, &\qquad
        &\Sigma_U^{(t)} = \sum_{i = 1}^N \eta' (A_U^{(t)}, \bm{B}_{U, i}^{(t)}), \\
        &B_V^{(t)} = J^T \hat{U}^{(t)} - \beta \Sigma_U^{(t)} \hat{V}^{(t)}, &\qquad
        &A_V^{(t)} = \beta \, \hat{U}^{(t)} \hat{U}^{{(t)}^T}, \\
        &\hat{\bm{V}}_j^{(t + 1)} = \eta (A_V^{(t)}, \bm{B}_{V, j}^{(t)}) \; \forall j, &\qquad
        &\Sigma_V^{(t + 1)} = \sum_{j = 1}^M \eta' (A_V^{(t)}, \bm{B}_{V, j}^{(t)}).
    \end{aligned}
\end{equation}

In order to adapt this algorithm to the online setting, we repeat procedure
(\ref{eq:miniamp}) and, as the $k$-th batch is processed, replace
calls to $\eta_{U} (A_U, \bm{B}_U)$ by
\begin{equation}
    \eta_{U} \Bigg(\underbrace{\sum_{\ell = 1}^{k - 1}
    A_{U, \ell}}_{\Lambda_{k - 1}} + A_{U, k}^{(t)}, \underbrace{\sum_{\ell = 1}^{k - 1}
    \bm{B}_{U, \ell}}_{\bm{\Theta}_{k, i}} + \bm{B}_{U, k}^{(t)} \Bigg).
\end{equation}
We assume that $U$ is fixed and $V_k$ changes for each batch $k$; thus, the calls to
$\eta_V$ do not change.

\section{State evolution and asymptotic limits}
\label{sec:appendix_se}
Through the state evolution equations, we analyze the behaviour and
performance of the algorithms described in the previous section. We
restrict ourselves to the Bayes-optimal case (i.e. the Nishimori line),
where the generative model is known. The strategy we use to go from the
offline to the streaming setting is easily adapted to the non-optimal
case.

\subsection{Generalized linear models}

The state evolution equations for a GLM with likelihood $P (y | z)$ and
prior $P_X (x)$ are
\begin{equation}
    \left\{
    \begin{aligned}
        &\hat{m}^{(t)} = -\alpha \, \mathbb{E}_{y, z, w} \; \partial_w g (y, w, 
        \rho - m^{(t)}), \\
        &m^{(t + 1)} = \mathbb{E}_{x, b} \; x \eta (\hat{m}^{(t)}, b),
    \end{aligned}
    \right.
\end{equation}
where we denote $\rho = \mathbb{E} x^2$, and the averages are taken with respect
to $P (x, b) = P_X (x) \mathcal{N} (b; \hat{m} x, \hat{m})$ and $P (y,
\omega, z) = P (y | z) \mathcal{N} (z; \omega, \rho - m) \mathcal{N}
(\omega; 0, m)$.
The MSE at each step is obtained from ${\cal E}^{(t)} = \rho - m^{(t)}$.
For a Gaussian likelihood $P (y | z) = \mathcal{N} (y; z, \Delta)$,
$\hat{m}^{(t)} = \frac{\alpha}{\Delta + {\cal E}^{(t)}}$ and we recover
(\ref{eq:se}).

The fixed points of the state evolution extremize the so-called replica free
energy
\begin{equation}
    \phi (m, \hat{m}) = \frac{1}{2} m \hat{m} - \mathbb{E}_{b, x} \log Z_x
    (\hat{m}, b) - \alpha \mathbb{E}_{y, \omega, z} \log Z_z (y, \omega,
    \rho - m)
\end{equation}
which gives the large system limit of the Bethe free energy (extremized by
AMP). The mutual information (\ref{eq:phi}) differs from $\phi (m) =
\operatorname*{extr}_{\hat{m}} \, \phi(m, \hat{m})$ by a constant -- more
specifically, by the entropy of $P (y | z)$, $i_\text{RS} (m) = \phi (m) -
\alpha H[P (y | z)]$. For a Gaussian likelihood, $H[P (y | z)] = \frac{1}{2}
\log(2 \pi e \Delta)$.

In order to adapt this to the streaming setting, we introduce
$\lambda_k^{(t)} = \sum_{\ell = 1}^{k - 1} \hat{m}_\ell^{(t_\text{max})} +
\hat{m}_k^{(t)}$ and iterate instead, for each mini-batch $k$
\begin{equation}
    \left\{
    \begin{aligned}
        &\lambda_k^{(t)} = \lambda_{k - 1} - \alpha_b \, \mathbb{E}_{y, z, w} \;
        \partial_w g (y, w, \rho - m_k^{(t)}), \\
        &m_k^{(t + 1)} = \mathbb{E}_{x, b_k} \; x \eta (\lambda_k^{(t)}, b),
    \end{aligned}
    \right.
    \label{eq:appendix_se}
\end{equation}
with the averages now computed over $P (x, b_k) = P_X (x) \mathcal{N} (b_k;
\lambda_k x, \lambda_k)$.
These equations should be iterated for $t = 1, \dots, t_\text{max}$, at
which point we assign $\lambda_k = \lambda_{k - 1}^{(t_\text{max})}$. The
MSE on $\bm{x}$ after mini-batch $k$ is processed is then given by ${\cal
E}_k = \rho - m_k^{(t_\text{max})}$.

Note that in the small batch size limit ($\alpha_b \to 0$), the equation for
$\lambda$ becomes an ODE
\begin{equation}
    \frac{d\lambda}{d\alpha} = -\mathbb{E}_{y, z, w} \; \partial_w
    g(y, w, \rho - m (\lambda)),
\end{equation}
which describes the performance of the ADF algorithm
\cite{opper_bayesian_1998,rossi_bayesian_2016}.

The free energy is also easily rewritten
\begin{equation}
    \phi_k (m_k, \hat{m}_k ; \lambda_{k - 1}) = \frac{1}{2} m_k \hat{m}_k -
    \mathbb{E}_{b, x} \log Z_x (\lambda_{k - 1} + \hat{m}_k, b) - \alpha_b
    \, \mathbb{E}_{y, \omega, z} \log Z_z (y, \omega, \rho - m_k),
\end{equation}
or analogously, by working with $\lambda_k = \lambda_{k - 1} + \hat{m}_k$ instead
of $\hat{m}_k$
\begin{equation}
    \phi_k (m_k, \lambda_k ; \lambda_{k - 1}) = \frac{1}{2} m_k (\lambda_k -
    \lambda_{k - 1}) - \mathbb{E}_{b, x} \log Z_x (\lambda_k, b) -
    \alpha_b \, \mathbb{E}_{y, \omega, z} \log Z_z (y, \omega, \rho - m_k).
\end{equation}
from which it is clear that the extrema of $\phi_k$ are given by the fixed
points of (\ref{eq:appendix_se}).

\subsubsection{Asymptotic behavior}

Equations (\ref{eq:appendix_se}) can be put in the following form
\begin{equation}
    \left\{
        \begin{aligned}
            {\cal E}_k &= \varepsilon(\lambda_k), \\
            \lambda_k &= \lambda_{k - 1} + \alpha_b \, \delta({\cal E}_k),
        \end{aligned}
    \right.
\end{equation}
where $\varepsilon(\lambda_k)$ and $\delta({\cal E}_k)$ are functions that
depend on the prior/channel respectively. Assuming $\varepsilon$ is
invertible, we rewrite this system of equations as a function of ${\cal
E}_k$ only
\begin{equation}
    \varepsilon^{-1} ({\cal E}_k) = \varepsilon^{-1} ({\cal E}_{k - 1}) +
    \alpha_b \, \delta({\cal E}_k),
\end{equation}
and then solve this equation for ${\cal E}_k$; that gives us a recurrence relation
which is unsolvable in most cases. We use instead asymptotic forms for
$\varepsilon$ and $\delta$, obtained in the $\lambda \to \infty$, ${\cal E}
\to 0$ limit. For the Bernoulli-Gaussian prior $P_0 (x_i) = \rho \,
\mathcal{N} (x_i; 0, 1) + (1 - \rho) \, \delta(x_i)$, we have
\begin{equation}
    \varepsilon(\lambda) \sim \frac{\rho}{\lambda},
\end{equation}
while a Gaussian likelihood gives, in the $\Delta \to 0$ limit,
$\delta({\cal E}) = \frac{1}{{\cal E}}$. Thus, for SLR
\begin{equation}
    \frac{\rho}{{\cal E}_k} \approx \frac{\rho}{{\cal E}_{k - 1}} +
    \frac{\alpha_\text{b}}{{\cal E}_k} \; \Rightarrow \; 
    {\cal E}_k \approx \bigg(1 - \frac{\alpha_b}{\rho}\bigg)^k
    {\cal E}_0,
\end{equation}
leading to (\ref{eq:asymptotic}). In the $\alpha_b \to 0$ limit we recover
the expression obtained by \cite{rossi_bayesian_2016}, $\operatorname{MSE}
(\alpha) \sim e^{-\frac{\alpha}{\rho}}$.

\subsection{Low-rank matrix factorization}

Also for low-rank models the large $N$ limit can be analyzed by taking into
account that $A_{U, V}^{(t)}$ and $B_{U, V}^{(t)}$ converge in distribution to
\cite{lesieur_mmse_2015}
\begin{equation}
    \begin{aligned}
        A_U^{(t)} &= \beta M_V^{(t)}, &\; \bm{B}_U &\sim \mathcal{N} (\beta
        M_V^{(t)} \bm{U}, \beta M_V^{(t)}), \\ 
        A_V^{(t)} &= \beta M_U^{(t)}, &\; \bm{B}_V &\sim \mathcal{N} (\beta
        M_U^{(t)} \bm{V}, \beta M_U^{(t)}),
    \end{aligned}
\end{equation}
where $M_U, M_V \in \mathbb{R}^{R \times R}$ are the overlap matrices
between the ground truth and the estimate at time $t$, that is
\begin{equation}
    \begin{aligned}
        &M_U^{(t)} = \mathbb{E}_{\bm{U}, \bm{B}_U} \,
            \bm{U}^T \eta_U (\beta M_V^{(t)}, \bm{B}_U), \\
        &M_V^{(t + 1)} = \alpha \, \mathbb{E}_{\bm{V}, \bm{B}_V} \,
            \bm{V}^T \eta_V (\beta M_U^{(t)}, \bm{B}_V).
    \end{aligned}
\end{equation}
While computing the expectations might become unfeasible for $R > 1$, an ansatz
of the following form can often be used \cite{lesieur_mmse_2015,lesieur_phase_2016}
\begin{equation}
    M_{U, V} = a_{U, V} \mathbb{I}_R + b_{U, V} \mathbb{J}_R,
\end{equation}
with $\mathbb{J}_R$ denoting the $R \times R$ matrix of ones. This
significantly simplifies the iteration above.

Again we adapt this to the online case by incrementing the matrices obtained
as each batch is processed, that is, we replace $M_V$ by
\begin{equation}
    \lambda_{V, k}^{(t + 1)} = \lambda_{V, k - 1} + \alpha_b \mathbb{E}_{\bm{V},
        \bm{B}_V} \, \bm{V}^T \eta_V (\beta M_{U, k}^{(t)}, \bm{B}_V).
\end{equation}
Note that since $U$ is fixed, the equation for $M_U$ does not change.

The replica free energy reads, in the offline case
\begin{equation}
    \phi (M_U, M_V) = \frac{\beta}{2} \operatorname{Tr} M_U M_V^T \, -
    \mathbb{E}_{\bm{U}, \bm{B}_U} \, \log Z_U (\beta M_V, \bm{B}_U) \,
    - \mathbb{E}_{\bm{V}, \bm{B}_V} \, \log Z_V (\beta M_U, \bm{B}_V),
\end{equation}
and we adapt it to the online case by taking into account that $\lambda_V$
is being incremented
\begin{equation}
    \begin{aligned}
        \phi (M_{U, k}, \lambda_{V, k}& ; \lambda_{V, k - 1}) = \frac{\beta}{2}
        \operatorname{Tr} M_{U, k} \, (\lambda_{V, k} - \lambda_{V, k - 1}) \, - \\
        &\mathbb{E}_{\bm{U}, \bm{B}_U} \, \log Z_U (\beta \lambda_{V, k}, \bm{B}_U)
        \, - \mathbb{E}_{\bm{V}, \bm{B}_V} \, \log Z_V (\beta M_{U, k}, \bm{B}_V).
    \end{aligned}
\end{equation}

\section{Performance for different number of iterations}
\label{sec:appendix_conv}
For our experiments in Figure \ref{fig:tradeoff}, Mini-AMP has been iterated
until, for each block, convergence is achieved -- that is, until 
$\frac{1}{N} \| \hat{\bm{x}}^{(t)} - \hat{\bm{x}}^{(t - 1)} \|_1 < 10^{-13}$.
Remarkably, our framework allows us to study the performance of the
algorithm even if we do not iterate it until convergence, but only for a few
steps $t_\text{max}$ instead. In Figure \ref{fig:tmax}, we investigate the
performance of Mini-AMP under the same settings of Figure \ref{fig:tradeoff}
(center), for different values of $t_\text{max}$. We observe that the
performance deteriorates if convergence is not reached.

\begin{figure}[ht]
    \centering
    \includegraphics[height=5.5cm]{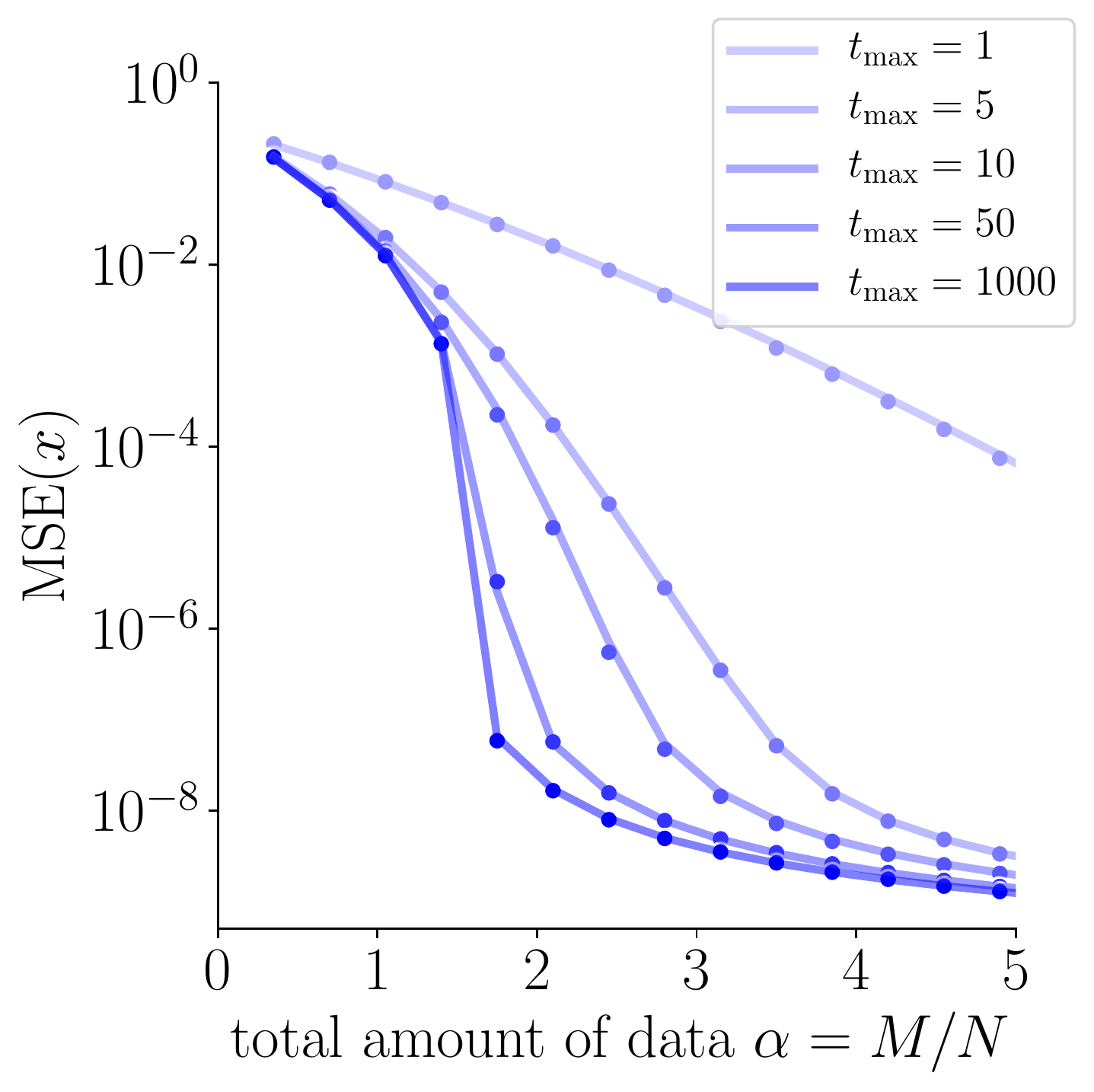}
    \caption{Performance of the Mini-AMP algorithm for different values of
    $t_\text{max}$, under the same settings of Figure \ref{fig:tradeoff}
    (center), and $\alpha_b = 0.35$. 
    Solid lines give state evolution, and
    symbols results of empirical experiments averaged over 10 realizations
    of size $N = 2000$. The performance deteriorates if the algorithm is
    not iterated until convergence.}
    \label{fig:tmax}
\end{figure}

\section{Experiments on real-world data}
\label{sec:appendix_realdata}
For the experiments with real data, we have used the following model
\begin{equation}
    P (U, V, Y) = \prod_{ij} \mathcal{N} (Y_{ij} ; \bm{U}_i \cdot \bm{V}_j,
    \Delta) \prod_{i = 1}^N \mathcal{N}_{\geq 0} (\bm{U}_i; \bm{0}, \sigma^2
    \mathbb{I}_R) \prod_{j = 1}^M \frac{1}{R} \sum_{k = 1}^R \delta (\bm{V}_j
    - \mathbbm{e}_k),
    \label{eq:glm}
\end{equation}
where
\begin{equation}
    \mathcal{N}_{\geq 0} (\bm{x}; \bm{\mu}, \sigma^2 \mathbb{I}_R) =
    \frac{1}{\mathcal{Z} (\bm{\mu}, \sigma^2)} \, \mathcal{N} (\bm{x};
    \bm{\mu}, \sigma^2 \mathbb{I}_R) \, \prod_{k = 1}^R \theta(x_k)
\end{equation}
is a truncated normal distribution supported on the positive
quadrant of a $R$-dimensional space, and $\mathcal{Z} (\bm{\mu}, \sigma^2)$
ensures proper normalization.

Note that evaluating the $\eta (A, \bm{B})$ function in this case is not
trivial, since it depends on the following integral
\begin{equation}
    \begin{aligned}
        Z_U (A, \bm{B}) &= \frac{1}{\mathcal{Z} (\bm{\mu}, \sigma^2)} \,
        \int d\bm{x} \, e^{-\frac{1}{2} \bm{x}^T A \bm{x} + \bm{B}^T \bm{x}}
        \, \prod_{k = 1}^R P_0 (x_k) \\
        &\propto \int \prod_{k = 1}^R dx_k \, P_0 (x_k) \, e^{-\frac{1}{2}
        A_{kk} x_k^2 + \big(B_k + \sum_{\ell \neq k} A_{k\ell} x_\ell\big)
        x_k}
    \end{aligned}
\end{equation}
where, in this case, $P_0 (x_k) = \mathcal{N} (x_k; \mu_k, \sigma^2) \,
\theta(x_k)$.  We proceed by performing a mean-field approximation. We first
define
\begin{equation}
    \tilde{\eta} (A, B) = \frac{\partial}{\partial B} \log \int dx \, P_0 (x) \,
    e^{-\frac{1}{2} A x^2 + Bx}.
\end{equation}
for scalar $A$ and $B$. Then, for each $i = 1, \dots, N$, we iterate
\begin{equation}
    \hat{U}_{ik} = \tilde{\eta} \big(A_{kk}, B_{ik} - \frac{1}{2} \sum_{\ell
    \neq k} A_{k\ell} \hat{U}_{i\ell}\big)
\end{equation}
sequentially in $k = 1, \dots, R$ until convergence is reached, at which
point we use the values obtained for assigning $\bm{\hat{U}}^{(t)}$ in AMP.
The variances are computed from
\begin{equation}
    \tilde{\sigma}_{ik} = \tilde{\eta}' \big(A_{kk},
    B_{ik} - \frac{1}{2} \sum_{\ell \neq k} A_{k\ell} \hat{U}_{i\ell}\big)
\end{equation}
and the covariance matrix used in AMP is obtained as a function of these variances
\begin{equation}
    \sigma_{k\ell} = \left\{
        \begin{aligned}
            &\sum_{i = 1}^N \tilde{\sigma}_{ik}, \quad&& \text{if $k = \ell$}, \\
            &-\frac{1}{2} A_{k\ell} \sum_{i = 1}^N \tilde{\sigma}_{ik}
                \tilde{\sigma}_{i\ell}, \quad&& \text{otherwise},
        \end{aligned}
    \right.
\end{equation}
where in order to assign the off-diagonal terms we have used a linear response
approximation, $\sigma_{k \ell} = \sum_{i = 1}^N \frac{\partial
\hat{U}_{ik}}{\partial B_{i \ell}}$.

We proceed by detailing other aspects of the experiments
\begin{description}
    \item[Initialization] At the first few mini-batches (usually the first
        five), we reinitialize the position of the centroids. We use the same
		strategy as the k-means++ algorithm \cite{arthur2007k}: 
		the first centroid is picked at random from the data points, and the
		next ones are sampled so as to have them far apart from each other. The
		labels are initialized according to the closest centroid.

    \item[Stopping criterion] For each batch, Mini-AMP was iterated either for 50 steps
        or until $\frac{1}{NR} \sum_{ik} |\hat{U}_{ik}^{(t)} -
        \hat{U}_{ik}^{(t - 1)}| + \frac{1}{MR} \sum_{jk} |\hat{V}_{jk}^{(t)} -
        \hat{V}_{jk}^{(t - 1)}| < 10^{-7}$.

    \item[Noise learning] We do not assign a fixed value for $\Delta$, but instead
        update it after each mini-batch is processed using a simple learning rule
        \begin{equation}
            \hat{\Delta}_k = \frac{1}{NM} \sum_{ij} (Y_{k,ij} - \bm{\hat{U}}_{k, i} \cdot \bm{V}_{k, j})^2.
        \end{equation}

	\item[Preprocessing] For MNIST, we work with all samples of digits 0, 1 and 2.
		They are rescaled so that the pixel intensities are between 0 and 1.
		For the 20 newsgroups dataset, we build Term Frequency Inverse
		Document Frequency (TF-IDF) features for 3 top-level hierarchies (\texttt{comp},
		\texttt{rec} and \texttt{sci}), and use the 1000 most frequent words;
		we rescale each feature vector so that its maximum is equal to 1.

	\item[Mini-batch K-means] We use the mini-batch K-means
		\cite{sculley_web-scale_2010} implementation available on scikit-learn
		\cite{scikit-learn}. Default parameters were used, apart from the
		centroids initialization, which was set to random normal variables of
		zero mean and variance $10^{-3}$ -- this seemed to improve the algorithm
		performance with respect to the standard choices.
\end{description}

\end{document}